\documentclass[11pt]{article}
\usepackage[dvips,letterpaper,margin=1in]{geometry}
\usepackage{sty}

\title{Lattice-Based Methods Surpass Sum-of-Squares in Clustering}

\author[a]{Ilias Zadik}
\author[b]{Min Jae Song}
\author[c]{Alexander S.\ Wein}
\author[b,d,e]{Joan Bruna}
\affil[a]{Department of Mathematics, Massachusetts Institute of Technology}
\affil[b]{Courant Institute of Mathematical Sciences, New York
  University}
\affil[c]{Simons Institute for the Theory of Computing, UC Berkeley}
\affil[d]{Center for Data Science, New York University}
\affil[e]{Center for Computational Mathematics, Flatiron Institute}

\begin{document}

\maketitle

\begin{abstract}

Clustering is a fundamental primitive in unsupervised learning which gives rise to a rich class of computationally-challenging inference tasks. In this work, we focus on the canonical task of clustering $d$-dimensional Gaussian mixtures with unknown (and possibly degenerate) covariance. Recent works (Ghosh et al.\ '20; Mao, Wein '21; Davis, Diaz, Wang '21) have established lower bounds against the class of low-degree polynomial methods and the sum-of-squares (SoS) hierarchy for recovering certain hidden structures planted in Gaussian clustering instances. Prior work on many similar inference tasks portends that such lower bounds strongly suggest the presence of an inherent statistical-to-computational gap for clustering, that is, a parameter regime where the clustering task is \textit{statistically} possible but no \textit{polynomial-time} algorithm succeeds.

One special case of the clustering task we consider is equivalent to the problem of finding a planted hypercube vector in an otherwise random subspace. We show that, perhaps surprisingly, this particular clustering model \textit{does not exhibit} a statistical-to-computational gap, despite the aforementioned low-degree and SoS lower bounds. To achieve this, we give an algorithm based on Lenstra--Lenstra--Lov\'asz lattice basis reduction which achieves the statistically-optimal sample complexity of $d+1$ samples. This result extends the class of problems whose conjectured statistical-to-computational gaps can be ``closed'' by ``brittle'' polynomial-time algorithms, highlighting the crucial but subtle role of noise in the onset of statistical-to-computational gaps.

\end{abstract}

\newpage
\tableofcontents
\newpage

\section{Introduction}
\label{sec:intro}

Many high-dimensional statistical inference problems exhibit a gap between what can be achieved by the optimal statistical procedure and what can be achieved by the best known \emph{polynomial-time} algorithms. As a canonical example, finding a planted $k$-clique in a $G(n,1/2)$ Erd\H{o}s--R\'{e}nyi random graph is statistically possible when $k$ exceeds $2 \log_2 n$ (via exhaustive search) but all known polynomial time algorithms require $k = \Omega(\sqrt{n})$, giving rise to a large conjectured ``possible but hard'' regime in between~\cite{jerrum-clique,alon-clique,sos-clique}. Such so-called \emph{statistical-to-computational gaps} are prevalent in many other key learning problems including sparse PCA (principal component analysis)~\cite{BR-reduction}, community detection~\cite{sbm-hard}, tensor PCA~\cite{MR-tensor-pca}, and random constraint satisfaction problems~\cite{sos-csp}, just to name a few. Unfortunately, since these are \emph{average-case} problems where the input is drawn from a specific distribution, current techniques appear unable to prove computational hardness in these conjectured ``hard'' regimes based on standard complexity assumptions such as $\mathsf{P} \ne \mathsf{NP}$.

Still, a number of different methods have emerged for understanding these gaps and providing ``rigorous evidence'' for computational hardness of statistical problems. Many involve studying the power of restricted classes of algorithms that are tractable to analyze, including statistical query (SQ) algorithms~\cite{kearnsSQ1998,sq-clique}, the sum-of-squares (SoS) hierarchy~\cite{parrilo-sos,lasserre-sos}, low-degree polynomial algorithms~\cite{HS-bayesian,hopkins2017power,hopkins2018thesis}, approximate message passing~\cite{amp}, MCMC methods~\cite{jerrum-clique}, and various notions of ``local'' algorithms~\cite{GS-local,zadikCOLT17,alg-tensor-pca}. As it turns out, the best known poly-time algorithms for a surprisingly wide range of statistical problems actually do belong to these restricted classes. As such, the above frameworks have been very successful at providing concrete explanations for statistical-to-computational gaps and allowing researchers to predict the location of the ``hard'' regime in new problems based on the location that such restricted classes of algorithms fail or succeed.

However, there are notorious exceptions where the above predictions turn out to be false. For example, the problem of learning parity (even in the absence of noise) is hard for SQ, SoS, and low-degree polynomials~\cite{kearnsSQ1998,grigoriev2001linear,schoenebeck2008linear}, yet actually admits a simple poly-time solution via Gaussian elimination. Yet, to the best of our knowledge, \emph{prior to the present work}, learning parities with no noise or other similar noiseless inference models based on linear equations, such as random 3-XOR-SAT, have been the only examples where some polynomial-time method (which appears to always be Gaussian elimination) works, while the SoS hierarchy and low-degree methods have been proven to fail. 

In this work, we identify a new \emph{class of problems} where the SoS hierarchy and low-degree lower bounds are provably bypassed by a polynomial-time algorithm. This class of problems is \emph{not} based on linear equations, and the suggested optimal algorithm is \emph{not} based on Gaussian elimination but on lattice basis reduction methods, which specifically seek to find a ``short'' vector in a lattice. Similar lattice-based method have over the recent years been able to ``close'' various statistical-to-computational gaps \cite{NEURIPS2018_ccc0aa1b, andoni2017correspondence, song2021cryptographic}, yet this is the first example we are aware of that they are able to ``close a gap'' where the SoS hierarchy is known to fail to do so.

The problems we analyze can be motivated from several angles in theoretical computer science and machine learning, and can be thought of as important special cases of well-studied problems such as Planted Vector in a Subspace, Gaussian Clustering, and Non-Gaussian Component Analysis. While our result is more general, one specific problem that we solve is the following: for a hidden unit vector $u \in \RR^d$, we observe $n$ independent samples of the form
\begin{equation}
\label{eq:hypercube-setup}
z_i \sim \mathcal{N}(x_i u, I_d - uu^\top), \quad i = 1,2,\ldots,n,
\end{equation}
where $x_i$ are i.i.d.\ uniform $\pm 1$, and the goal is to recover the hidden signs $x_i$ and the hidden direction $u$ (up to a global sign flip). Prior to our work, the best known poly-time algorithm required $n \gg d^2$ samples\footnote{Here are throughout, the notation $\gg$ hides logarithmic factors.}~\cite{planted-vec-ld}. Furthermore, this was believed to be unimprovable 
due to lower bounds against SoS algorithms and low-degree polynomials~\cite{MS-sos2,KB-sos4,lifting-sos,affine-planes,tim-sos6,planted-vec-ld,unknown-cov}. Nevertheless, we give a poly-time algorithm under the much weaker assumption $n \ge d+1$. In fact, this sample complexity is essentially optimal for the previous recovery problem, as shown by our information-theoretic lower bound (see Section~\ref{sec:it-lower-bound}). 
Our result makes use of the Lenstra--Lenstra--Lov\'asz (LLL) algorithm for lattice basis reduction~\cite{lenstra1982factoring}, a powerful algorithmic paradigm that has seen recent, arguably surprising, success in solving to information-theoretic optimality a few different ``noiseless'' statistical inference problems, some even in regimes where it was conjectured that no polynomial-time method works: Discrete Regression \cite{NEURIPS2018_ccc0aa1b,LLL_TIT}, Phase Retrieval \cite{andoni2017correspondence, song2021cryptographic}, Cosine Neuron Learning \cite{song2021cryptographic}, and Continuous Learning with Errors \cite{bruna2020continuous,song2021cryptographic}\footnote{in the exponentially-small noise regime}. Yet, to the best of our knowledge, this work is the first to establish the success of an LLL-based method in a regime where low-degree and SoS lower bounds both suggest computational intractability. This raises the question of whether LLL can ``close'' any other conjectured statistical-to-computational gaps. We believe that understanding the power and limitations of the LLL approach is an important direction for future research towards understanding the computational complexity of inference.

We also point out one weakness of the LLL approach: our algorithm is brittle to the specifics of the model, and relies on the observations being ``noiseless'' in some sense. For instance, our algorithm only solves the model in~\eqref{eq:hypercube-setup} because the $x_i$ values lie \emph{exactly} in $\pm 1$ and the covariance $\Sigma = I-uu^\top$ has quadratic form $u^\top \Sigma u$ \emph{exactly} equal to zero (or, similarly to other LLL applications \cite{NEURIPS2018_ccc0aa1b}, of exponentially small magnitude). If we were to perturb the model slightly, say by adding an inverse-polynomial amount of noise to the $x_i$'s, our algorithm would break down because of the known non-robustness properties of the LLL algorithm. In fact, a noisy version (with inverse-polynomial noise) of one problem that we solve is the \emph{homogeneous} Continuous Learning with Errors problem (hCLWE), which is provably hard based on the standard assumption~\cite[Conjecture 1.2]{micciancio2009lattice} from lattice-based cryptography that certain worst-case lattice problems are hard against quantum algorithms~\cite{bruna2020continuous}. All existing algorithms for statistical problems based on LLL suffer from the same lack of robustness. In this sense, there is a strong analogy between LLL and the other known successful polynomial-time method for noiseless inference, namely the Gaussian elimination approach to learning parity: both exploit very precise algebraic structure in the problem and break down in the presence of even a small (inverse-polynomial) amount of noise.

As discussed above, our results ``break'' certain computational lower bounds based on SoS and low-degree polynomials. Still, we believe that these types of lower bounds are interesting and meaningful, but some care should be taken when interpreting them. It is in fact already well-established that such lower bounds can sometimes be beaten on ``noiseless'' problems (a key example being Gaussian elimination). 
However, there are some subtleties in how ``noiseless'' should be defined here, and whether fundamental problems with statistical-to-computational gaps such as planted clique---which has implications for many other inference problems via average-case reductions (e.g.~\cite{BR-reduction,MW-reduction,HWX-pds,BBH-reduction})---should be considered ``noiseless.'' We discuss these issues further in Section~\ref{sec:noiseless}.

\subsection{Main contributions}

The main inference setting we consider in this work is as follows. Let $d \in \mathbb{N}$ be the growing ambient dimension the data lives in, $n$ be the number of samples, and $a>0$ be what we coin as the \emph{spacing} parameter. We consider arbitrary labels on the one-dimensional lattice: $a x_i$ where $x_i \in \mathbb{Z}$ for $i=1,\ldots,n$, under the weak constraints $|x_i| \leq 2^d,\, i=1,\ldots,n$ and $d^{-O(1)} \leq |a| \leq d^{O(1)}$. We also consider an arbitrary direction $u \in \mathcal{S}^{d-1}$ and an unknown covariance matrix $\Sigma$ with $\Sigma u=0$ and ``reasonable'' spectrum, in the sense that $\Sigma$ does not have exponentially small or large eigenvalues in directions orthogonal to $u$ (see Assumption \ref{assum:u-weak_sep}). In particular, the choice of $\Sigma=I-uu^{\top}$ is permissible per our assumptions but much more generality is possible. Our goal is to learn \emph{both} the labels $x_i,\, i=1,\ldots,n$ and the hidden direction $u$ (up to global sign flip applied to both $x$ and $u$) from independent samples
\begin{align}\label{eq:separation-setup} 
z_i \sim \mathcal{N}(a x_i u, \Sigma),\quad i=1,\ldots,n.
\end{align}

\paragraph{Exact recovery with polynomial-time algorithm.} 
Our main algorithmic result is informally stated as follows.

\begin{theorem}[Informal statement of Theorem \ref{thm:LLL}]
\label{thm:inf}
 Under the above setting, if $n=d+1$ then there is an LLL-based algorithm (Algorithm \ref{alg:lll}) which terminates in polynomial time and outputs exactly, up to a global sign flip, both the correct labels $x_i$ and the correct hidden direction $u$ with probability $1-\exp(-\Omega(d)).$
 
\end{theorem}

\begin{table}[t]\label{table}
    \centering
    \begin{tabular}{|c|c|c|c|c|}
    \hline
    \textbf{Problems} & \textbf{Low-degree LB} & \textbf{SoS LB} & \textbf{Previous Best} & \textbf{Our Results} \\
    \hline\hline
    \makecell{Planted Vector \\ (Rademacher)} & $\tilde\Omega(d^2)$ & $\tilde\Omega(d^{3/2})$& $\tilde{O}(d^2)$~\cite{planted-vec-ld,unknown-cov} & $d+1$ \\
    \hline
    \makecell{Gaussian Clustering \\ ($\mathrm{SNR}=\infty$)} & $\tilde\Omega(d^2)$ & $\tilde\Omega(d^{3/2})$ & $\tilde{O}(d^2)$~\cite{unknown-cov} & $d+1$ \\
    \hline
    \makecell{hCLWE \\ (Noiseless)} & - & - & ${O}(d^2)$~\cite{bruna2020continuous} & $d+1$ \\
    \hline
    \end{tabular}
    \caption{Sample complexity upper and lower bounds for \emph{polynomial-time} exact recovery for Planted Hypercube Vector Recovery, Gaussian Clustering, and hCLWE. LB stands for ``Lower Bound''. }
    \label{tab:my_label}
\end{table}

Now, as explained in Section~\ref{sec:intro} our theorem has algorithmic implications for three previously studied problems: Planted Vector in a Subspace, Gaussian Clustering, and hCLWE, which is an instance of Non-Gaussian Component Analysis. In all three settings, the previous best algorithms required $\Omega(d^2)$ samples (formally this is true for the dense case of the planted vector in a subspace setting, but $\omega(d)$ samples are required for many sparse settings as well). As explained, in many of these cases lower bounds have been achieved for the classes of low-degree methods and the SoS hierarchy. In this work, we show that LLL can surpass these lower bounds and succeed with $n=d+1$ samples in all three problems. We provide more context in Section \ref{sec:related-work} and exact statements in Section \ref{sec:lll-implication}. A high-level description of our contributions can be found in Table \ref{tab:my_label}.

\paragraph{Information-theoretic lower bound for exact recovery of the hidden direction.}

One can naturally wonder whether for our setting there is something even better than LLL that can be achieved with bounded computational resources or even unbounded ones. We complement the previous result with an information-theoretic lower bound (Theorem~\ref{thm:infIT}) showing that \emph{no} estimation procedure can succeed at exact recovery of the hidden direction $u$ using at most $n = d-1$ samples. This means LLL is information-theoretically optimal for recovering the hidden direction up to at most one additional sample. 

\begin{theorem}[Informal statement of Theorem \ref{IT_param}]
\label{thm:infIT}
Under the setting of~\eqref{eq:hypercube-setup}, if $n \leq d-1$, there is no estimation procedure that can guarantee with probability greater than $1/2$ exact recovery of the hidden direction $u$ up to a global sign flip.
\end{theorem}

In fact, our formal theorem (Theorem~\ref{IT_param}) shows that recovering the labels $\{x_i\}_{i=1}^n$ is strictly \emph{easier} than recovering the hidden direction $\bu$ in the sense that $d-1$ samples are insufficient for determining $\bu$ even when we have exact knowledge of the labels $\{x_i\}_{i=1}^n$. In light of this fact, it is natural to ask about the sample complexity of recovering the labels. We provide an answer to this question by showing that no estimator can recover the labels $\{x_i\}_{i=1}^n$ (up to a global sign flip) with probability $1-o(1)$ when $n=\lfloor \rho d \rfloor$ for any constant $\rho \in (0,1)$ (Theorem~\ref{thm:infITlabel}). This implies that our sample complexity of $n=d+1$ is optimal up to a $1+o(1)$ factor for label recovery.

\begin{theorem}[Informal statement of Theorem \ref{thm:label_rec}]
\label{thm:infITlabel}
Let $\rho \in (0,1)$ be a fixed constant. Under the setting of~\eqref{eq:hypercube-setup}, if $n \le \lfloor \rho d \rfloor$, there is no estimation procedure that can guarantee with probability $1-o(1)$ exact recovery of the labels $\{x_i\}_{i=1}^n$ up to a global sign flip.
\end{theorem}

\subsection{Relation to prior work}
\label{sec:related-work}

\paragraph{Non-gaussian component analysis.}

Non-gaussian component analysis (NGCA) is the problem of identifying a non-gaussian direction in random data. Concretely, this is a generalization of~\eqref{eq:hypercube-setup} where the $x_i$ are drawn i.i.d.\ from some distribution $\mu$ on $\RR$. When $\mu$ is the Rademacher $\pm 1$ distribution, we recover the problem in~\eqref{eq:hypercube-setup} as a special case.

The NGCA problem was first introduced in~\cite{blanchard2006search}, inspiring a long line of algorithmic results~\cite{kawanabe2006estimating,sugiyama2008approximating,diederichs2010sparse,diederichs2013sparse,bean2014non,sasaki2016sufficient,nordhausen2017asymptotic,vempala2011structure,tan2018polynomial,goyal2019non}. This problem has also played a key role in many hardness results in the statistical query (SQ) model: starting from the work of~\cite{sq-robust}, various special cases of NGCA have provided SQ-hard instances for a variety of learning tasks~\cite{sq-robust,diakonikolas2018list,diakonikolas2019efficient,diakonikolas2021statistical,bubeck2019adversarial,goel2020statistical,diakonikolas2020near,diakonikolas2021optimality,goel2020superpolynomial,diakonikolas2020algorithms,diakonikolas2020hardness,spiked-transport}. More recently, a special case of NGCA has also been shown to be hard under the widely-believed assumption that certain worst-case lattice problems are hard~\cite{bruna2020continuous}. NGCA is a special case of the more general \emph{spiked transport model}~\cite{spiked-transport}.

Our main result (Theorem \ref{thm:inf}) solves NGCA with only $n = d+1$ samples in the case where $\mu$ is supported arbitrarily on an exponentially large subset of a 1-dimensional discrete lattice. While this case is essentially the noiseless, equispaced version of the ``parallel Gaussian pancakes'' problem, which was first introduced and shown to be SQ-hard by~\cite{sq-robust}, our result does not bypass known SQ lower bounds~\cite{sq-robust,bruna2020continuous} as the hard construction involves Gaussian pancakes with non-negligible ``thickness''. 

A \textit{concurrent and independent work} by Diakonikolas and Kane \cite{diakonikolas2021nongaussian} proposed a very similar LLL-based polynomial-time algorithm to ours for the case where $\mu$ is ``nearly'' supported on a finite subset of a finitely generated additive subgroup of $\mathbb{R}$, which includes~\eqref{eq:hypercube-setup} as a special case. Interestingly, while their algorithm provably works with $n=d+1$ samples in the noiseless case, it also tolerates a small exponential-in-$d$ level of noise in the labels $x_i$ at the expense of using $n=2d$ samples. The exact noise tolerance of our proposed algorithm is left as an interesting open question.

\paragraph{Planted vector in a subspace.}

A line of prior work has studied the problem of finding a ``structured'' vector planted in an otherwise random $d$-dimensional subspace of $\RR^n$, where $d < n$. A variety of algorithms have been proposed and analyzed in the case where the planted vector is \emph{sparse}~\cite{demanet2014scaling,barak2014rounding,qu2016finding,sos-spectral,qu2020finding}. One canonical Gaussian generative model for this problem turns out to be equivalent to NGCA (with the same parameters $d,n$), where the entrywise distribution of the planted vector corresponds to the non-gaussian distribution $\mu$ in NGCA. More specifically, the subspace in question is the \emph{column span} of the matrix whose \emph{rows} are the NGCA samples $z_i$; see e.g.\ Lemma~4.21 of~\cite{planted-vec-ld} for the formal equivalence.

Motivated by connections to the \emph{Sherrington--Kirkpatrick model} from spin glass theory, the setting of a planted \emph{hypercube} (i.e.\ $\pm 1$-valued) vector in a subspace has received recent attention; this is equivalent to the problem in~\eqref{eq:hypercube-setup}. Specifically, sum-of-squares (SoS) lower bounds have been given for refuting the existence of a hypercube vector in (or close to) a purely random subspace. First, it was shown that when $n = O(d)$, SoS relaxations of degree 2~\cite{MS-sos2}, 4~\cite{KB-sos4,lifting-sos}, and 6~\cite{tim-sos6} fail. A later improvement~\cite{affine-planes} shows failure of degree-$n^{\Omega(1)}$ SoS when $n \ll d^{3/2}$ and conjectures that this condition can be improved to $n \ll d^2$ (see Conjectures~8.1 and 8.2 of~\cite{affine-planes}).

The state-of-the-art algorithmic result for recovering a planted vector in a subspace is~\cite{planted-vec-ld}, which builds on~\cite{sos-spectral} and in particular analyzes a spectral method proposed by~\cite{sos-spectral}. For a planted $\rho$-sparse Rademacher vector ($\rho n$ entries are nonzero, and these nonzero entries are $\pm 1/\sqrt{\rho}$), this spectral method succeeds at recovering the vector provided $n \gg \rho^2 d^2$~\cite{planted-vec-ld}. On the other hand, if $n \ll \rho^2 d^2$ then all low-degree polynomial algorithms fail, implying in particular that \emph{all} spectral methods (in a large class) fail~\cite{planted-vec-ld}. These results cover the special case of a planted hypercube vector ($\rho = 1$), in which case there is a spectral method that succeeds when $n \gg d^2$, and failure of low-degree and spectral methods when $n \ll d^2$.

The above results suggest inherent computational hardness of planted sparse Rademacher vector when $n \ll \rho^2 d^2$. However, perhaps surprisingly, our main result (Theorem \ref{thm:inf}) implies that this problem can actually be solved via LLL in polynomial time whenever $n \ge d+1$. Thus, LLL beats all low-degree algorithms whenever $\rho \gg 1/\sqrt{d}$. However, our algorithm requires the entries of the planted vector to \emph{exactly} lie in $\{0,\pm 1/\sqrt{\rho}\}$, whereas the spectral method of~\cite{sos-spectral,planted-vec-ld} succeeds under more general conditions.

We remark that the planted hypercube vector problem is closely related to the \emph{negatively-spiked Wishart model} with a hypercube spike, which can be thought of as a model for generating the \emph{orthogonal complement} of the subspace. The work of~\cite{sk-cert} gives low-degree lower bounds for this negative Wishart problem and conjectures hardness when $n = O(d)$ (Conjecture~3.1). However, our results do not refute this conjecture because the conjecture is for a slightly noisy version of the problem (since the SNR parameter $\beta$ is taken to be strictly greater than $-1$).

\paragraph{Clustering.}

Our model (\ref{eq:hypercube-setup}) is an instance of a broader clustering problem (\ref{eq:separation-setup}) under Gaussian mixtures. In the binary case, it consists of $n$ i.i.d.\ samples $\{(z_i, x_i)\}_{i=1 \ldots n} \in \mathbb{R}^{d} \times \{-1,+1\}$ of the Gaussian mixture 
$P(x_i = -1) = P(x_i = +1) = 1/2$, and $z_i \,|\, x_i \sim \mathcal{N}(x_i u, \Sigma)$, with unknown mean $u$ and covariance $\Sigma$. The goal of clustering is to infer the mixture variables $\{x_i\}$ from the observations $\{z_i\}$.
Clustering algorithms have been analysed extensively, both from the statistical and computational perspective. 

The statistical performance is driven by the signal-to-noise ratio $\mathrm{SNR}=v^\top \Sigma^{-1} v$, in the sense that the error rate for recovering the mixture labels is $\exp(-\Omega(\mathrm{SNR}))$ \cite{friedman1989regularized}. Exact recovery of the vector of $n$ labels is thus possible only when $\mathrm{SNR} \gtrsim \log(n)$.

Recently, \cite{unknown-cov} showed that the MLE estimator for the missing labels corresponds to a Max-Cut problem, which recovers the solution when $n = \tilde{\Omega}(d)$. Moreover, the authors argued that while $\textrm{SNR}$ drives the inherent statistical difficulty of the estimation problem, a relaxed quantity $\textrm{S}=\|v\|^2/\|\Sigma\|$ presumably controls the computational difficulty. 
In particular, the largest such gap is attained in the covariance choice of  (\ref{eq:hypercube-setup}), for which $\textrm{SNR}=\infty$ while $\mathrm{S}=1$. 
In this regime, they identified a gap between the statistical and computational performance of multiple existing algorithms, raising the crucial question whether such guarantees can be obtained using polynomial-time algorithms. Several previous works \cite{brubaker2008isotropic, moitrav2010mixture, bakshi2020robustly, cai2019chime, flammarion2017robust} introduce algorithms that either require larger sample complexity $n = \tilde{\Omega}(d^2)$, or have non-optimal error rates, for instance based on k-means relaxations \cite{royer2017adaptive, mixon2017clustering, giraud2019partial,li2020birds}.
Leveraging existing SoS lower bounds on the associated Max-Cut problem (see Section \ref{sec:noiseless}), \cite{unknown-cov} suggest a statistical-to-computational gap for exact recovery in the binary Gaussian mixture. Our main result (Theorem \ref{thm:inf}) implies that this problem can be solved via the LLL basis reduction method in polynomial time whenever $n \ge d+1$. Thus  in the present work, we refute this conjecture for $\mathrm{SNR}=\infty$ under a weak ``niceness'' assumption on the covariance matrix $\Sigma$.

\paragraph{LLL-based statistical recovery.}

Our algorithm is based on the breakthrough use of LLL for solving average-case subset sum problems in polynomial-time, specifically the works of \cite{Lagarias85} and \cite{FriezeSubset}. In these works, it is established that while the (promise) Subset-Sum problem is NP-hard, for some integer-valued distributions on the input weights it becomes polynomial-time solvable by applying the LLL basis reduction algorithm on a carefully designed lattice. Building on these ideas, \cite{NEURIPS2018_ccc0aa1b, LLL_TIT} proposed a new algorithm for noiseless discrete regression and discrete phase retrieval which provably solves these problems using only one sample, surpassing previous local-search lower bounds based on the so-called Overlap Gap Property \cite{zadikCOLT17}. Again using the Subset-Sum ideas the LLL approach has also ``closed'' the gap for noiseless phase retrieval \cite{andoni2017correspondence, song2021cryptographic} which was conjectured to be hard because of the failure of approximate message passing in this regime \cite{maillard2020phase}. Furthermore, for the problem of noiseless Continuous LWE (CLWE), the LLL algorithm has been shown to succeed with $n=\Omega(d^2)$ samples in \cite{bruna2020continuous}, and later in \cite{song2021cryptographic} with the information-theoretically optimal $n=d+1$ samples.

Our work adds a perhaps important new conceptual angle to the power of LLL for noiseless inference. A common feature of all the above inference models where LLL has been successfully applied is that they fall into the class of generalized linear models (GLMs). A GLM is generally defined as follows: for some hidden direction $w \in \mathcal{S}^{d-1}$ and ``activation'' function $\phi: \mathbb{R} \rightarrow \mathbb{R}$ one observes $n$ i.i.d.\ samples of the form $y_i=\phi(\langle X_i,w \rangle)+\xi_i,\, i=1,\ldots,n$ where $X_i \in \mathbb{R}^d$ and $\xi_i \in \mathbb{R}$ are i.i.d.\ random variables. Our work shows how to successfully apply LLL and achieve statistically optimal performance for the clustering setting \eqref{eq:separation-setup}, which importantly does not admit a GLM formulation. We consider this a potentially interesting conceptual contribution of the present work, since many ``hard'' inference settings, such as the planted clique model, also do not belong in the class of GLMs.

\subsection{``Noiseless'' problems and implications for SoS/low-degree lower bounds}
\label{sec:noiseless}

\paragraph{SoS and low-degree lower bounds.}

The sum-of-squares (SoS) hierarchy~\cite{parrilo-sos,lasserre-sos,sos-csp,sos-clique} (see~\cite{barak2016proofs,sos-survey,fleming2019semialgebraic} for a survey) and low-degree polynomials~\cite{HS-bayesian,hopkins2017power,hopkins2018thesis} (see~\cite{kunisky2019notes} for a survey) are two restricted classes of algorithms that are often studied in the context of statistical-to-computational gaps. These are not the only two such frameworks, but we will focus on these two because our result ``breaks'' lower bounds in these two frameworks. SoS is a powerful hierarchy of semidefinite programming relaxations. Low-degree polynomial algorithms are simply multivariate polynomials in the entries of the input, of degree logarithmic in the input dimension; notably, these can capture all \emph{spectral methods} (subject to some technical conditions), i.e., methods based on the leading eigenvalue/eigenvector of some matrix constructed from the input (see Theorem~4.4 of~\cite{kunisky2019notes}). Both SoS and low-degree polynomials have been widely successful at obtaining the best known algorithms for a wide variety of high-dimensional ``planted'' problems, where the goal is to recover a planted signal buried in noisy data. While there is no formal connection between SoS and low-degree algorithms, they are believed to be roughly equivalent in power~\cite{hopkins2017power}. It is often informally conjectured that SoS and/or low-degree methods are as powerful as the best poly-time algorithms for ``natural'' high-dimensional planted problems (nebulously defined). As a result, lower bounds against SoS and/or low-degree methods are often considered strong evidence for inherent computational hardness of statistical problems.

\paragraph{Issue of noise-robustness.}

In light of the above, it is tempting to conjecture optimality of SoS and/or low-degree methods among all poly-time methods for a wide variety of statistical problems. While this conjecture seems to hold up for a surprisingly long and growing list of problems, there are, of course, limits to the class of problems for which this holds. As discussed previously, a well-known counterexample is the problem of learning parity (or the closely-related XOR-SAT problem), where Gaussian elimination succeeds in a regime where SoS and low-degree algorithms provably fail. This counterexample is often tossed aside by the following argument: ``Gaussian elimination is a brittle algebraic algorithm that breaks down if a small amount of noise is added to the labels, whereas SoS/low-degree methods are more robust to noise and are therefore capturing the limits of poly-time \emph{robust} inference, which is a more natural notion anyway. If we restrict ourselves to problems that are sufficiently \emph{noisy} then SoS/low-degree methods should be optimal.'' However, we note that in our setting, SoS/low-degree methods are strictly suboptimal for a problem that \emph{does} have plenty of Gaussian noise; the issue is that the signal and noise have a particular joint structure that preserves certain exact algebraic relationships in the data. This raises an important question: what exactly makes a problem ``noisy'' or ``noiseless'', and under what kinds of noise should we believe that SoS/low-degree methods are unbeatable? In the following, we describe one possible answer.

\paragraph{The low-degree conjecture.}

The ``low-degree conjecture'' of Hopkins~\cite[Conjecture 2.2.4]{hopkins2018thesis} formalized one class of statistical problems for which low-degree polynomials are believed to be optimal among poly-time algorithms. These are certain \emph{hypothesis testing} problems where the goal is to decide whether the input was drawn from a null (i.i.d.\ noise) distribution or a planted distribution (containing a planted signal). In our setting, one should imagine testing between $n$ samples drawn from the model~\eqref{eq:hypercube-setup} and $n$ samples drawn i.i.d.\ from $\mathcal{N}(0,I_d)$. Computational hardness of hypothesis testing generally implies hardness of the associated recovery/estimation/learning problem (which in our case is to recover $x$ and $u$) as in Theorem~3.1 of~\cite{planted-vec-ld}. The class of testing problems considered in Hopkins' conjecture has two main features: first, the problem should be highly symmetric, which is typical for high-dimensional statistical problems (although Hopkins' precise notion of symmetry does not quite hold for the problems we consider in this paper). Second, and most relevant to our discussion, the problem should be \emph{noise-tolerant}. More precisely, Hopkins' conjecture states that if low-degree polynomials fail to distinguish a null distribution $\QQ$ from a planted distribution $\PP$, then no poly-time algorithm can distinguish $\QQ$ from \emph{a noisy version of} $\PP$. For our setting, the appropriate ``noise operator'' to apply to $\PP$ (which was refined in~\cite{HW-counter}) is to replace each sample $z_i$ by
\[ \sqrt{1-\delta^2} z_i + \delta z_i' \]
where $z_i' \sim \mathcal{N}(0,I_d)$ independently from $z_i$, for an arbitrarily small constant $\delta > 0$. This has the effect of replacing $x_i$ with $\sqrt{1-\delta^2}x_i + \delta \tilde{z}_i$ where $\tilde{z}_i \sim \mathcal{N}(0,1)$. This noise is designed to ``defeat'' brittle algorithms such as Gaussian elimination, and indeed our LLL-based algorithm is also expected to be defeated by this type of noise.

To summarize, the problem we consider in this paper is \emph{not} noise-tolerant in the sense of Hopkins' conjecture because the Gaussian noise depends on the signal (specifically, there is no noise in the direction of $u$) whereas Hopkins posits that the noise should be \emph{oblivious} to the signal. Thus, in hindsight we should perhaps not be too surprised that LLL was able to beat SoS/low-degree for this problem. In other words, our result does not falsify the low-degree conjecture or the sentiment behind it (low-degree algorithms are optimal for noisy problems), with the caveat that one must be careful about the precise meaning of ``noisy.'' We feel that this lesson carries an often-overlooked conceptual message that may have consequences for other fundamental statistical problems such as planted clique, which we discuss next.

\paragraph{Planted clique.}

The \emph{planted clique conjecture} posits that there is no polynomial-time algorithm for distinguishing between a random $G(n,1/2)$ graph and a $G(n,1/2)$ graph with a clique planted on $k$ random vertices (by adding all edges between the clique vertices), when $k \ll \sqrt{n}$. The planted clique conjecture is central to the study of statistical-to-computational gaps because it has been used as a primitive to deduce computational hardness of many other problems via a web of average-case reductions (e.g.~\cite{BR-reduction,MW-reduction,HWX-pds,BBH-reduction}). A refutation of the planted clique conjecture would be a major breakthrough that could cast doubts on whether other statistical-to-computational gaps are ``real'' or whether the gap can be closed by a better algorithm. As a result, it is important to ask ourselves why we believe the planted clique conjecture. Aside from the fact that it has resisted all algorithmic attempts so far, the primary concrete evidence for the conjecture comes in the form of lower bounds against SoS and low-degree polynomials~\cite{sos-clique,hopkins2018thesis}. However, it is perhaps unclear whether these should really be thought of as strong evidence for inherent hardness because (like the problem we study in the paper) planted clique is~\emph{not} noise-tolerant in the sense of Hopkins' conjecture (discussed above). Specifically, the natural noise operator would be to independently resample a small constant fraction of the edges, which would destroy the clique structure. In other words, the conjecture of Hopkins only implies that a noisy variant of planted clique (namely \emph{planted dense subgraph}) is hard when $k \ll \sqrt{n}$.

While we do not have any concrete reason to believe that LLL could be used to solve planted clique, we emphasize that planted clique is in some sense a ``noiseless'' problem and so we do not seem to have a principled reason to conjecture its hardness based on SoS and low-degree lower bounds. On the other hand, we should perhaps be somewhat more confident in the ``planted dense subgraph conjecture'' because planted dense subgraph is a truly noisy problem in the sense of~\cite[Conjecture 2.2.4]{hopkins2018thesis}.

\section{Preliminaries}
\label{sec:lll_background}

The key component of our algorithmic results is the LLL lattice basis reduction algorithm. The LLL algorithm receives as input $d$ linearly independent vectors  $v_1,\ldots, v_d \in \mathbb{Z}^d$ and outputs an integer linear combination of them with ``small'' $\ell_2$ norm. Specifically, let us define the lattice generated by $d$ \emph{integer} vectors as simply the set of integer linear combination of these vectors.

\begin{definition}[Lattice]
\label{dfn_lattice}
Given linearly independent $v_1,\ldots, v_d \in \mathbb{Z}^d$, let 
\begin{align} L=L(v_1,\ldots,v_d)=\left\{\sum_{i=1}^{d} \lambda_i v_i : \lambda_i \in \mathbb{Z}, i=1,\ldots,d \right\}\;, 
\end{align} 
which we refer to as the lattice generated by integer-valued $v_1,\ldots,v_d$. We also refer to $(v_1,\ldots,v_d)$ as an (ordered) basis for the lattice $L$.
\end{definition}

The LLL algorithm solves a search problem called the approximate Shortest Vector Problem (SVP) on a lattice $L$, given a basis of it.
\begin{definition}[Shortest Vector Problem]
\label{def:gamma-approx-short}
An instance of the algorithmic  $\alpha$-approximate SVP for a lattice $L \subseteq \mathbb{Z}^d$ is as follows. Given a lattice basis $v_1,\dots,v_d \in \mathbb{Z}^d$ for the lattice $L$, find a vector $\widehat{x}\in L$, such that
\begin{align*}
    \|\widehat{x}\|_2 \leq \alpha \cdot \mu(L)\;.
\end{align*}
where $\mu(L) = \min_{x \in L, x\neq 0} \|x\|_2$.
\end{definition}
The following theorem holds for the performance of the LLL algorithm, whose details can be found in  \cite{lenstra1982factoring} or~\cite{lovasz1986algorithmic}.
\begin{theorem}[{\cite{lenstra1982factoring}}]
\label{LLL_thm_original}
There is an algorithm (namely the LLL lattice basis reduction algorithm), which receives as input a basis for a lattice $L$ given by $v_1,\ldots,v_d \in \mathbb{Z}^d$ which
\begin{itemize}
    \item[(1)] returns a vector $v \in L$ satisfying $\|v\|_2 \le 2^{d/2} \mu(L)$,
    \item[(2)] terminates in time polynomial in $d$ and $\log \left(\max_{i=1}^d \|v_i\|_{\infty}\right).$
\end{itemize}
\end{theorem}

In this work, we use the LLL algorithm for an integer relation detection application, a problem which we formally define below.

\begin{definition}[Integer relation detection]
An instance of the \emph{integer relation detection problem} is as follows. Given a vector $b=(b_1,\dots,b_k)\in\mathbb{R}^k$, find an $m \in\mathbb{Z}^k\setminus\{{\bf 0}\}$, such that $\langle b,m \rangle := \sum_{j=1}^k b_jm_j=0$. In this case, $m$ is said to be an integer relation for the vector $b$.
\end{definition}

To define our class of problems, we make use of the following two standard objects.
\begin{definition}[Bernoulli--Rademacher vector]
\label{defn:bernoulli-rademacher}
We say that a random vector $v \in \mathbb{R}^n$ is a Bernoulli--Rademacher vector with parameter $\rho \in (0, 1]$ and write $v \sim \BR(n, \rho)$, if the entries of $v$ are i.i.d.\ with 
\begin{align*}
v_i = 
\begin{cases}
0 & \text{ with probability } 1 - \rho , \\
1/\sqrt{n \rho} & \text{ with probability } \rho/2 , \\
-1/\sqrt{n \rho} & \text{ with probability } \rho/2 . 
\end{cases}
\end{align*}
\end{definition}

\begin{definition}[Discrete Gaussian on $s\mathbb{Z}$]
\label{defn:discrete-gaussian}
Let $r,s > 0$ be real numbers. We define the discrete Gaussian distribution with \emph{width} $r$ supported on the scaled integer lattice $s\mathbb{Z}$ to be the distribution whose probability mass function at each $x \in s\mathbb{Z}$ is proportional to $\exp\left(-\frac{x^2}{2r^2}\right)$.
\end{definition}

The following tail bound on the discrete Gaussian will be useful in Section~\ref{sec:lll-implication}, in which we reduce the hCLWE model (also known as ``Gaussian pancakes'') (Model~\ref{mod:hCLWE}) to our general model (Model~\ref{mod:general_model}) which is the central problem for our LLL-based algorithm.

\begin{claim}[{Adapted from~\cite[Claim I.6]{song2021cryptographic}}]
\label{claim:discrete-gaussian-finite-approx}
Let $\gamma \ge 1$ be a real number, and let $\nu$ be the discrete Gaussian of width 1 supported on $(1/\gamma)\mathbb{Z}$ such that the probability mass function at $x \in (1/\gamma)\mathbb{Z}$ is given by
\begin{align*}
    \nu(x) = \frac{1}{\sZ}\exp(-x^2/2)\;,
\end{align*}
where $\sZ = \sum_{x \in (1/\gamma)\mathbb{Z}} \exp(-x^2/2)$ is the normalization constant. Then, the following bound holds.
\begin{align}
\label{eqn:normalization-constant-bound}
    \sZ \ge \gamma\sqrt{2\pi}\left(1-2(1+1/(4\pi \gamma)^2)\exp(-2\pi^2\gamma^2)\right) > 1\;.
\end{align}
In particular, for $t \ge \gamma$,
\begin{align*}
    \Pr_{x \sim \nu}[|x| \ge t] \le 4\exp(-t^2/2)\;.
\end{align*}
\end{claim}

\begin{proof}
The first inequality in \eqref{eqn:normalization-constant-bound} follows from \cite[Claim 1.6]{song2021cryptographic}. The fact that $\sZ > 1$ follows from our assumption that $\gamma \ge 1$. Since $\sZ > 1$, we have
\begin{align*}
    \Pr_{x \sim \nu}[|x| \ge t] = \frac{1}{\sZ}\sum_{x \in (1/\gamma)\mathbb{Z}, |x| \ge t} \exp(-x^2/2) \le 2\sum_{x \in (1/\gamma)\mathbb{Z}, x \ge t} \exp(-x^2/2)\;.
\end{align*}

Now notice that the terms of the series are decaying in a geometric fashion for $x \ge \gamma$ since
\begin{align*}
    \exp\left(-\frac{(x+1/\gamma)^2}{2}\right)/\exp\left(-\frac{x^2}{2}\right) = \exp\left(-\frac{2x+1/\gamma}{2\gamma}\right) \le \exp(-1) < 1/2\;.
\end{align*}
It follows that
\begin{align*}
    \Pr_{x\sim \nu}[|x|\ge t] \le 2\sum_{x \in (1/\gamma)\mathbb{Z}, x \ge t} \exp(-x^2/2) \le 4\exp(-t^2/2)\;.
\end{align*}
\end{proof}

\section{The LLL-based algorithm}
\label{sec:mainLLL}
We now present the main contribution of this work, which is an LLL-based polynomial-time algorithm that provably solves the general problem defined in Model~\ref{mod:general_model} with access to only $n=d+1$ samples.

We deal formally with samples coming from $d$-dimensional Gaussians, which have as their mean some unknown multiple of an unknown unit vector $u \in \mathcal{S}^{d-1}$, and also some unknown covariance $\Sigma$ which nullifies $u$ and satisfies the following weak ``separability'' condition.

\begin{assumption}[Weak separability of the spectrum]
\label{assum:u-weak_sep}
Fix a unit vector $u \in \mathcal{S}^{d-1}$. We say that a positive semi-definite $\Sigma \in \mathbb{R}^{d \times d}$ is $u$-weakly separable if for some constant $C>0$ it holds that 
\begin{itemize}
    \item[(a)] $\Sigma u=0$ and,
    \item[(b)] All other eigenvalues of $\Sigma$ lie in the interval $[d^{-C},d^C].$
\end{itemize}
\end{assumption}
Notice that in particular the canonical case $\Sigma=I-uu^{\top}$ is $u$-weakly separable as all eigenvalues of $\Sigma$ are equal to one, besides the zero eigenvalue which has multiplicity one and eigenvector $u.$

Under the weak separability assumption we establish the following generic result.
\begin{model}[Our general model]
\label{mod:general_model}
Let $d, n \in \mathbb{N}$, known spacing level $a>0$ satisfying $d^{-c} \leq a \leq d^c$ for some constant $c>0,$ and arbitrary $x_i \in \mathbb{Z} \cap[-2^d,2^d],\, i=1,\ldots,n$. Consider also an arbitrary $u \in \mathcal{S}^{d-1}$ and an arbitrary unknown $\Sigma \in \mathbb{R}^{d \times d}$ which is $u$-weakly separable per Assumption \ref{assum:u-weak_sep}. Conditional on $u, \Sigma$ and $\{x_i\}_{i=1,\ldots,n}$, we then draw independent samples $z_1,\ldots,z_n \in \mathbb{R}^d$ where $z_i \sim \sN((ax_i) u, \Sigma)$. The goal is to use $z_i,\, i=1,\ldots,n$ to recover both $\{x_i\}_{i=1,\ldots,n}$ and $u$ up to a global sign flip, with probability $1-\exp(-\Omega(d))$ over the samples $z_i,\, i=1,\ldots,n.$
\end{model}

In the following section we discuss the guarantee of the our proposed Algorithm. After, we discuss how our results implies that learning in polynomial time with $d+1$ samples is possible for: (1) the ``planted sparse vector'' problem (Model~\ref{mod:planted-sparse-vector}), (2) the homogeneous Continuous Learning with Errors (hCLWE) problem, also informally called the ``Gaussian pancakes'' problem (Model~\ref{mod:hCLWE}), and finally (3) Gaussian clustering (Model~\ref{mod:gaussian clustering}).
As mentioned earlier in Section~\ref{sec:intro}, our result bypasses known lower bounds for SoS and low-degree polynomials which suggested that the regime $n=\Theta(d)$ would not be achievable by polynomial-time algorithms.

In what follows and the rest of the paper, for some $N \in \mathbb{N}$ and $x \in \mathbb{R}$ we denote by $(x)_N:=2^{-N} \lfloor 2^N x \rfloor$ the truncation of $x$ to its first $N$ bits after zero.

\subsection{The algorithm and the main guarantee}

\begin{algorithm}[t]
\caption{LLL-based algorithm for recovering $u, (x_i)_{i=1,\ldots,d+1}$ }
\label{alg:lll}
\KwIn{$n=d+1$ samples $z_i \in \mathbb{R}^d,\, i=1, \ldots, d+1,$ spacing $a>0.$}
\KwOut{Estimated labels $\hat{x}_i \in \mathbb{Z},\, i=1,\ldots,d+1$ and unit vector $\hat{u} \in S^{d-1}$.}
\vspace{1mm} \hrule \vspace{1mm}

Construct a $d \times d$ matrix $Z$ with columns $z_2,\ldots,z_{d+1}$, and let $N=\lceil d^4(\log d)^2\rceil$.\\
\If{$\det(Z) = 0$}{\Return $\hat{u} = 0$ and output FAIL.}
Compute $\lambda_1 = 1$ and $\lambda_i = \lambda_i(z_1,\ldots,z_{d+1})$ given by $(\lambda_2,\ldots,\lambda_{d+1})^\top = -Z^{-1} z_1$. \\
Set $M=2^{2d}$ and $\tilde{v}=\left((\lambda_2)_N,\ldots,(\lambda_{d+1})_N, 2^{-N}\right) \in \mathbb{R}^{d+1}$.\\
Output $(t_1,t_2)\in  \ZZ^{d+1} \times \ZZ$ from running the LLL basis reduction algorithm on the lattice generated by the columns of the following $(d+2)\times (d+2)$ integer-valued matrix $B$,
\begin{align*}
B= \left(\begin{array}{@{}c|c@{}}
  \begin{matrix}
  M 2^N(\lambda_1)_N 
  \end{matrix}
  & M 2^N \tilde{v} \\
\hline
  0_{(d+1)\times 1} &
  I_{(d+1)\times (d+1)}
\end{array}\right).
  \end{align*}
  \\ 
  $\hat{u}_0 \leftarrow \mathrm{SolveLinearEquation}(u', Z^\top u' = at_1)$. \\

\If{$\hat{u}_0=0$}{\Return $\hat{u} = 0$ and output FAIL.}
Set $\hat{x}_i=(t_1)_i/\|\hat{u}_0\|_2, i=1,\ldots,d+1.$\\
\Return $\hat{x}_i,\, i=1,\ldots,d+1$ and $\hat{u}_0/\| \hat{u}_0\|_2$ and output SUCCESS.
\end{algorithm}

Our proposed algorithm for solving Model \ref{mod:general_model} is described in Algorithm \ref{alg:lll}. Specifically we assume the algorithm receives $n=d+1$ independent samples according to Model \ref{mod:general_model}. As we see in the following theorem, the algorithm is able to recover exactly (up to a global sign flip) both the hidden direction $u$ and the hidden labels $x_i,\, i=1,\ldots,n$ in polynomial time.

\begin{theorem}
\label{thm:LLL}
Algorithm \ref{alg:lll}, given as input independent samples $(z_i)_{i=1,\ldots,d+1}$ from Model \ref{mod:general_model} with hidden direction $u$, covariance $\Sigma,$ and true labels $\{x_i\}_{i =1,\ldots,d+1}$ satisfies the following with probability $1-\exp(-\Omega(d))$: there exists $\eps \in \{-1,1\}$ such that the algorithm's outputs $\{\hat{x}_i\}_{i=1,\ldots,d+1}$ and $\hat{u} \in S^{d-1}$ satisfy
\begin{align*}
    \hat{x}_i&= \epsilon x_i\;\, \text{ for } i=1,\ldots,d+1\\
   \text{ and  } \hat{u} & =\epsilon u\;.
\end{align*} 
Moreover, Algorithm \ref{alg:lll} terminates in $\poly(d)$ steps.

\end{theorem}

We now provide intuition behind the algorithm's success. Note that for the unknown $u$ and $x_i$ it holds that
\begin{align}
\label{ln_eqs}
\langle z_i, u \rangle =a x_i \quad\text{for all } i=1,\ldots,d+1;.
\end{align}
In the first step, the algorithm checks a certain general-position condition on the received samples, which naturally is satisfied almost surely for our random data. In the following crucial three steps, the algorithm attempts to recover only the hidden integer labels $x_i$ without learning $u$. To do this, it exploits a certain random integer linear relation that the labels $x_i$'s satisfy which \textit{importantly does not involve any information about the unknown $u$}, besides its existence.
The key observation leading to this relation is the following.
Since we have $d+1$ vectors $z_i$ in a $d$-dimensional space, there exist scalars $\lambda_1,\ldots,\lambda_{d+1}$ (depending on the $z_i$'s) such that $\sum_{i=1}^{d+1}\lambda_i z_i=0.$ These are exactly the $\lambda_i$'s that the algorithm computes in the second step. Using them, observe that the following linear equation holds, due to \eqref{ln_eqs},
\begin{align}
    &\sum_{i=1}^{d+1} \lambda_i  a x_i=\sum_{i=1}^{d+1} \lambda_i \langle z_i, u\rangle = \left\langle \sum_{i=1}^{d+1} \lambda_i z_i, u\right \rangle = \langle 0, u \rangle = 0\;, \label{eqn:key-observation}
\end{align}
and therefore since $a>0$ it gives the following integer linear equation
\begin{align}
\label{eq:hidden}
    \sum_{i=1}^{d+1} \lambda_i  x_i=0.
\end{align} 

Again note that the $\lambda_i$'s can be computed from the given samples $z_i$, so this is an equation whose sole unknowns are the labels $x_i$.  With this integer linear equation in mind, the algorithm in the following step employs the powerful LLL algorithm applied to an appropriate lattice. This application of the LLL is based on the breakthrough works of \cite{lagarias1984knapsack,FriezeSubset} for solving random subset-sum problems in polynomial-time, as well as its recent manifestations for  solving various other noiseless inference settings such as binary regression \cite{NEURIPS2018_ccc0aa1b} and phase retrieval \cite{andoni2017correspondence,song2021cryptographic}. To get some intuition for this connection, notice that in the case $x_i \in \{-1,1\}$, \eqref{eq:hidden} is really a (promise) subset-sum relation with weights $\lambda_i$ and unknown subset $\{i: x_i=1\}$ for which the corresponding $\lambda_i$'s sum to $\frac{1}{2} \sum_{i=1}^{d+1} \lambda_i$. Now, after some careful technical work, including an appropriate truncation argument to work with integer-valued data, and various anti-concentration arguments such as the Carbery--Wright anticoncentration toolkit \cite{Carbery2001DistributionalAL}, one can show that the LLL step indeed recovers a constant multiple of the labels $x_i,\, i=1,\ldots,d+1$ with probability $1-\exp(-\Omega(d))$ (see also the next paragraph for more details on this). At this point, it is relatively straightforward to recover $u$ using the linear equations~\eqref{ln_eqs}.

Now we close by presenting the key technical lemma which ensures that LLL recovers the hidden labels $x_i$ by finding a ``short'' vector in the lattice defined by the columns of the matrix $B$ in Algorithm \ref{alg:lll}. Notice that if truncation at $N$ bits was not present, that is we were ``allowed'' to construct the lattice basis with the non-integer numbers $\lambda_i$ instead of $(\lambda_i)_N$, then a direct calculation based on \eqref{eq:hidden} would give that the hidden labels are embedded in an element of the lattice simply because we would have $$B(x_1,\ldots,x_{d+1},0)^{\top}=(0,x_1,\ldots,x_{d+1})^{\top}.$$ As this ``hidden vector'' in the lattice is $M$-independent (and $M$ is taken to be very large) this naturally suggests that this vector may be ``short'' compared to the others in the lattice. The following lemma states that  with probability $1-\exp(-\Omega(d))$, this is indeed the case. The random lattice generated by the columns of $B$ indeed does not contain any ``spurious'' short vectors other than the vector of the hidden labels and, naturally, its integer multiples. This implies that the LLL algorithm, despite its $2^{d/2}$ approximation ratio, will indeed return the integer relation that is ``hidden in'' the $z_i$'s.

\begin{lemma}[No spurious short vectors]
\label{lem:min_norm} 
Let $ d \in \mathbb{N}, $ $a \in [d^{-c},d^c]$ for some constant $c>0$ and $N=\lceil d^4(\log d)^2 \rceil$. Let $u \in S^{d-1}$ be an arbitrary unit vector, $\Sigma \in \mathbb{R}^{d \times d}$ an arbitrary unknown $u$-separable matrix, and let $x_i \in \mathbb{Z}\cap [-2^d,2^d]$ for $i=1,\ldots,d+1$ be arbitrary but not all zero. Moreover, let $\{z_i\}_{i=1,\ldots,d+1}$ be independent samples from $\mathcal{N}((ax_i)u,\Sigma)$, and let $B$ be the matrix constructed in Algorithm \ref{alg:lll} using $\{z_i\}_{i=1,\ldots,d+1}$ as input and $N$-bit precision. Then, with probability $1-\exp(-\Omega(d))$ over the samples, for any $t = (t_1, t_2) \in \mathbb{Z}^{d+1} \times \mathbb{Z}$ such that $t_1$ is not an integer multiple of $x=(x_1,\ldots,x_{d+1})$, the following holds:
\begin{align*}
    \|Bt\|_2 > 2^{2d}\;.
\end{align*}
\end{lemma}

The proof of Lemma \ref{lem:min_norm} is in Section \ref{sec:proof_lem_LLL}.

\section{Implications of the success of LLL}
\label{sec:lll-implication}
        
The algorithmic guarantee described in Theorem \ref{thm:LLL} has interesting implications for various recently studied problems. 

\subsection{The planted sparse vector problem} As we motivated in the introduction, a model of notable recent interest is the so-called ``planted sparse vector'' setting, where one plants a sparse vector in an unknown subspace.
\begin{model}[Planted sparse vector]
\label{mod:planted-sparse-vector}
Let $d, n \in \mathbb{N}$ be such that $d \le n$, and let the sparsity level $\rho \in (0,1]$. First draw $v \sim \BR(n, \rho)$ and $R \in \mathbb{R}^{d \times d}$ be a uniformly at random chosen orthogonal matrix. Then, we sample i.i.d.\ $\tilde{z}_1,\ldots,\tilde{z}_{d-1} \sim \sN(0,(1/n)I_n)$. Denote by $\tilde{Z} \in \mathbb{R}^{n \times d}$ the matrix whose columns we perceive as generating the ``hidden subspace'':
\begin{align}
\label{eqn:planted-gaussian-basis}
\tilde{Z} = \begin{bmatrix} v, \tilde{z}_1,\ldots,\tilde{z}_{d-1} \end{bmatrix}\,.
\end{align}
The statistician observes the rotated matrix $Z = \tilde{Z}R$. The goal is, given access to $Z$, to recover the $\rho$-sparse vector $v\in \mathbb{R}^n$.
\end{model}

As explained in Section~\ref{sec:intro}, all low-degree polynomial methods, and therefore all spectral methods (in a large class) fail for this model when $n \ll \rho^2 d^2$ \cite{planted-vec-ld}. Furthermore, in the dense case where $\rho=1$ and $n \ll d^{3/2}$, even degree-$n^{\Omega(1)}$ SoS algorithms are known to fail \cite{affine-planes}.

Yet, using our Theorem \ref{thm:LLL} we can prove that for any $\rho \in (0,1]$ the LLL algorithm solves Model \ref{mod:planted-sparse-vector} with $n=d+1$ samples in polynomial-time. In particular, this improves the state-of-the-art for this task, and provably beats all low-degree algorithms, when $\rho=\omega(1/\sqrt{d})$. The corollary is based on the equivalence (up to rescaling) between Model \ref{mod:planted-sparse-vector} and some appropriate sub-case of Model \ref{mod:general_model}~\cite{unknown-cov}.

\begin{corollary}
Suppose $ n, d \in \mathbb{N}$ with $d \leq n$ and $d \rightarrow +\infty.$ Also let $\rho \in (0,1]$ scale as $\rho=\omega(1/n)$. Given $n=d+1$ samples $z_i,\, i=1,\ldots,d+1$ from Model \ref{mod:planted-sparse-vector}, Algorithm \ref{alg:lll} with input $\sqrt{d+1}z_i,\, i=1,\ldots,d+1$ and spacing $a=1/\sqrt{\rho},$ terminates in polynomial time and outputs $v$ (up to a global sign flip) with probability $1-\exp(-\Omega(\min\{d,\rho n\}))$. 
\end{corollary}

\begin{proof}

As established in \cite[Lemma 4.21]{planted-vec-ld} for $z_i$ a sample from Model \ref{mod:planted-sparse-vector} from a setting where the planted sparse vector is $v$, the distribution of $\sqrt{n}z_i=\sqrt{d+1}z_i$ is a sample from Model \ref{mod:general_model} with hidden direction $v$, spacing $a=1/\sqrt{\rho}$ and covariance $\Sigma=I-vv^{\top}$. Note that the covariance $\Sigma$ satisfies Assumption \ref{assum:u-weak_sep} for straightforward reasons and same for the condition on the spacing as $n=d+1$ and $1>\rho>1/n$. On top of this, since $\rho n \rightarrow +\infty$ the output of $v\sim \BR(n, \rho)$ is not the zero vector with probability $1-(1-\rho)^n=1-\exp(-\Omega(\rho n)).$ Hence, combining the above and Theorem \ref{thm:LLL} we immediately conclude the desired result.
\end{proof}

\subsection{Continuous Learning with Errors} 
\label{sec:CLWE}
As we motivated in the introduction, a second model that is of interest, due to its connections to fundamental problems in lattice-based cryptography, is the homogeneous continuous learning with errors (hCLWE).

\begin{model}[hCLWE]
\label{mod:hCLWE}
Let $\gamma > 0$ be a real number and let $\nu$ be a discrete Gaussian of width 1 supported on $(1/\gamma)\mathbb{Z}$. Let $d, n \in \mathbb{N}$.  First draw a random unit vector $u \sim \mathcal{S}^{d-1}$ and i.i.d.\ $x_1,\ldots,x_n \sim \nu$. Conditional on $u$ and $\{x_i\}_{i=1}^n$, draw independent samples $z_1,\ldots,z_n \in \mathbb{R}^d$ where $z_i \sim \sN(x_i u, I_d - uu^\top)$. The statistician observes the matrix $Z \in \mathbb{R}^{n \times d}$ with rows $z_1^\top, \dots, z_n^\top$.  The goal is, given access to $Z$, to recover the unit vector $u \in \mathbb{R}^n$. 
\end{model}

It is known that for $\gamma \ge 2\sqrt{d}$, if we add any inverse-polynomial Gaussian noise in the hidden direction $u$ of hCLWE, even detecting the existence of such discrete structure is hard, under standard worst-case hardness assumptions from lattice-based cryptography~\cite{bruna2020continuous}. Moreover, there are SQ lower bounds, which are unconditional, for this noisy version~\cite{sq-robust,bruna2020continuous}. As a direct corollary of our Theorem \ref{thm:LLL} we show that in the noiseless case where no Gaussian noise is added, LLL can recover exactly the hidden direction in polynomial time with $d+1$ samples. We remark that while~\cite{bruna2020continuous} claim that their LLL-based algorithm for (inhomogeneous) CLWE could be generalized to the hCLWE setting, their algorithm uses $O(d^2)$ samples, which is suboptimal.

\begin{corollary}
Suppose $d \in \mathbb{N}$ with $d \rightarrow +\infty.$ Assume that $\gamma=\mathrm{poly}(d).$ Given $n=d+1$ samples $z_i,\, i=1,\ldots,d+1$ from Model \ref{mod:hCLWE}, Algorithm \ref{alg:lll} with input $z_i,\, i=1,\ldots,d+1$ and spacing $a=1/\gamma,$ terminates in polynomial time and outputs $v$ (up to a global sign flip) with probability $1-\exp(-\Omega(d)).$
\end{corollary}

\begin{proof}
The proof follows from Claim \ref{claim:discrete-gaussian-truncation-reduction}.
\end{proof}

\begin{claim}[hCLWE reduces to the general model]
\label{claim:discrete-gaussian-truncation-reduction}
Let $d \in \mathbb{N}$ and $\gamma = \poly(d)$. Then, with probability $1-\exp(-\Omega(d))$, $n=\poly(d)$ random samples $\{z_i\}_{i=1}^n$ from Model~\ref{mod:hCLWE} (Gaussian pancakes) with parameter $\gamma$ is a valid instance of Model~\ref{mod:general_model} (our general model) for $\Sigma=I-vv^{\top}$.
\end{claim}
\begin{proof}
Let $\nu$ denote the discrete Gaussian supported on $(1/\gamma)\mathbb{Z}$ in Model~\ref{mod:hCLWE}. By Claim~\ref{claim:discrete-gaussian-finite-approx}, for any $t \ge \gamma$, it holds that
\begin{align*}
    \Pr_{x \sim \nu}[|x| \ge t] \le 4\exp(-t^2/2)\;.
\end{align*}
Let $t=\max\{\sqrt{d},\gamma\}$. Then, $\Pr_{x\sim \nu}[|x| \ge t] = \exp(-\Omega(d))$. By a straightforward union bound, with probability $1-\exp(-\Omega(d))$ the true ``labels'' $\{x_i\}_{i=1}^{n}$ of the samples $\{z_i\}_{i=1}^{n}$ satisfy $|x_i| \le t = \poly(d)$ for all $i = 1,\ldots,n$, and thus this is a valid instance of Model~\ref{mod:general_model}.

\end{proof}

\subsection{Gaussian Clustering} Finally, a third model which has been extensively studied in the theoretical machine learning community is the (Bayesian) Gaussian clustering model.

\begin{model}[Gaussian clustering]
\label{mod:gaussian clustering}
Let $d,n \in \mathbb{N}$. Fix some unknown positive semi-definite matrix $\Sigma \in \mathbb{R}^{d \times d}.$ Now draw a random unit vector $u \sim \mathcal{S}^{d-1}$ and i.i.d.\ uniform Rademacher labels \ $x_1,\ldots,x_n \sim \{-1,1\}$. Conditional on $u$ and $\{x_i\}_{i=1}^n$, draw independent samples $z_1,\ldots,z_n \in \mathbb{R}^d$ where $z_i \sim \sN(x_i u, \Sigma)$. The statistician observes the matrix $Z \in \mathbb{R}^{n \times d}$ with rows $z_1^\top, \dots, z_n^\top$.
The goal is, given the observation matrix $Z \in \mathbb{R}^{n \times d}$, to recover (up to a global sign flip) the labels $x_i,\, i=1,\ldots,n$.
\end{model}

As explained in the introduction, the recent work by \cite{unknown-cov} shows that for Model \ref{mod:gaussian clustering}, exact reconstruction of the labels is possible as long as $u^\top\Sigma^{-1/2}u>2\log n$ if $\Sigma$ is invertible, and of course also in the regime where $\Sigma u=0$. The authors of \cite{unknown-cov} show that for any such $\Sigma$ it is possible to achieve exact reconstruction with $\tilde{O}(d)$ samples by some computationally inefficient method, and construct and analyze computationally efficient methods that work with $\tilde{O}(d^2)$ samples. On top of this, they conjecture that the regime $\Omega(d)=n=o(d^2)$ is computationally hard based on various forms of rigorous evidence such as failure of SoS and low-degree methods. Notably, the failure of the SoS hierarchy transfers even to the case $\Sigma u=0$ in the regime $n \le \tilde{\Omega}(d^{3/2})$. As a direct corollary of Theorem \ref{thm:LLL}, we refute their conjecture for any covariance matrix $\Sigma$ which nullifies $u$ and satisfies a weak ``niceness" assumption, namely Assumption \ref{assum:u-weak_sep}. We show that under this assumption, exact reconstruction of the labels (and therefore of the clusters) is possible in polynomial time with only $n=d+1$ samples. 

\begin{corollary}
Suppose $d \in \mathbb{N}$ with $d \rightarrow +\infty.$  Given $n=d+1$ samples $z_i,\, i=1,\ldots,d+1$ from Model \ref{mod:gaussian clustering} with arbitrary covariance $\Sigma$ which is $u$-weakly separable per Assumption \ref{assum:u-weak_sep}. Then Algorithm \ref{alg:lll} with input $z_i,\, i=1,\ldots,d+1$ and spacing $a=1$ terminates in polynomial-time and outputs exactly the labels $x_i,\, i=1,\ldots,n$ (up to a global sign flip) with probability $1-\exp(-\Omega(d)).$
\end{corollary}

\paragraph{The finite $\mathrm{SNR}$ regime.} 
Our algorithmic guarantees for the linear sample complexity regime crucially depend on the ``noiseless'' aspect of the model, translated into our weak separability assumption (Assumption \ref{assum:u-weak_sep}), which corresponds to $\mathrm{SNR}=\infty$ in the notation of~\cite{unknown-cov}.   An important question is to understand whether the linear sample complexity guarantees are possible in polynomial-time for finite (albeit growing with dimension)
$\mathrm{SNR}$. When $\mathrm{SNR}=\exp(d^{\Theta(1)})$, the similar LLL-based procedure by~\cite{diakonikolas2021nongaussian} is proven to work and we do expect that our algorithm remains successful in this regime as well.
    On the other hand, from \cite{unknown-cov} we know that $\mathrm{SNR} > 2 \log d$ is necessary in order to ensure exact recovery. This leaves open a wide range of $\mathrm{SNR}$. While we do not provide an answer to this question here, we want to note that the $\mathrm{SNR}=\poly(d)$ regime would require some novel ideas, at least in terms of the analysis of the suggested algorithms. Both our algorithm and the algorithm by~\cite{diakonikolas2021nongaussian} are analyzed in a way that establishes success for both Gaussian Clustering (Rademacher labels) and hCLWE (Discrete Gaussians) simultaneously. Yet, the known computational hardness of hCLWE \cite{bruna2020continuous} strongly suggests that hCLWE cannot be solved with inverse polynomial noise by any polynomial-time algorithm with polynomially many samples. Hence, both the analyses of the currently proposed algorithms for Gaussian Clustering with $\Theta(d)$ samples are expected to fail in this regime. We leave it as an interesting open question whether there exists a polynomial-time algorithm that succeeds in the Gaussian Clustering model for $\mathrm{SNR} \ll \exp(d)$ and linear sample complexity.

\section{Information-theoretic lower bounds}
\label{sec:it-lower-bound}

In this section we establish information-theoretic lower bounds associated with problem \eqref{eq:hypercube-setup} both for \emph{parameter recovery}, i.e., recovering the hidden direction $u$, as well as \emph{label recovery}, i.e., recovering the hidden labels $\{x_i\}_{i=1}^n$. These lower bounds show that the sample complexity $n=d+1$ for our LLL-based algorithm is optimal in the following sense: for exact recovery of the hidden direction $u$ (up to a global sign flip), even $n=d-1$ samples are not sufficient. For exact recovery of the labels (up to a global sign flip), $n=(1-o(1))d$ samples are required.

\subsection{Optimality of \texorpdfstring{$d+1$}{d+1} samples for parameter recovery}
In this section we establish that when $a=1$, $n \leq d-1$, and $x_i \in \{-1,1\},$ for $i=1,\ldots,n$, one cannot information-theoretically exactly recover the hidden direction $u \in \mathcal{S}^{d-1}$ from $n$ independent samples $z_i \sim \mathcal{N}(x_i u,I-uu^{\top})$. This establishes the optimality of our LLL approach which achieves a much more generic guarantee in Theorem \ref{thm:LLL} in recovering $u$ exactly, up to one additional sample compared to the information-theoretic lower bound.

In fact, our lower bound is somewhat stronger. We assume that the statistician also has exact knowledge of the signs $x_i \in \{-1,1\}$. We show that even in this setting exact recovery of the hidden direction $u$ with probability greater than $1/2$ using at most $d-1$ samples is impossible. Notice that if $x_i$'s are known to the statistician, there is no global sign ambiguity with respect to $u$. 

\begin{theorem}[Parameter recovery]
\label{IT_param}
Let arbitrary $x_i \in \{-1,1\},\, i=1,\ldots,d-1$ be fixed and known to the statistician. Moreover, let $u \in \mathcal{S}^{d-1}$ be a uniformly random unit vector. For each $i=1,\ldots, d-1$ let $z_i$ be a sample generated independently from $\mathcal{N}(x_i u,I-uu^{\top}).$ Then it is information-theoretically impossible to construct an estimate $\hat{u}=\hat{u}(\{(x_i,z_i)\}_{i=1}^n) \in \mathcal{S}^{d-1}$ which satisfies $\hat{u}= u$ with probability larger than $1/2$.
\end{theorem}

\subsection{Information-theoretic lower bound for label recovery}

In this section we establish that when $n = \lfloor \rho d\rfloor$ for any fixed $\rho \in (0,1)$ and $x_i \in \{-1,1\},\, i=1,\ldots,n$, one cannot information-theoretically exactly recover the latent variables $\{x_i\}$ from independent samples $z_i \sim \mathcal{N}(x_i u,I-uu^{\top})$, where $u \in \mathcal{S}^{d-1}$ is the unknown direction. Hence, combined with Theorem~\ref{thm:LLL}, we can conclude that our LLL approach achieves the optimal information-theoretic sample complexity for label recovery up to a $1+o(1)$ factor.

\begin{theorem}[Label recovery]
\label{thm:label_rec}
Let $\rho \in (0,1)$ be a fixed constant and let $n = \lfloor \rho d \rfloor$.
Let $x \in \{-1,1\}^n$ be drawn uniformly at random, and $u \in \mathcal{S}^{d-1}$ be a uniformly random unit vector. For each $i=1,\ldots, n$ let $z_i \in \mathbb{R}^{d}$ be a sample generated independently from $\mathcal{N}(x_i u,I-uu^{\top}).$ Then no estimator can exactly recover the labels $\{x_i\}_{i=1}^n$ up to a global sign flip from the observations $\{z_i\}_{i=1}^n$ with probability $1-o(1)$.
\end{theorem}

The main idea of the proof, which can be found in Section \ref{sec:label-recovery-lb}, is to first compute the conditional density (Lemma~\ref{lem:posterioradmiss}), which we denote by $f(Z|X=x)$, of the observations $Z \in \mathbb{R}^{n \times d}$ given labels $x \in \{-1,1\}^n$. We achieve this by viewing Model \ref{mod:gaussian clustering} as the limit $\sigma \to 0$ of the Gaussian clustering model with covariance $\Sigma = \sigma^2 uu^\top + (I - uu^\top)$, and then identifying small perturbations in the labels $x$ that result in small relative fluctuations in the conditional density $f(Z|X=x)$ that crucially remain uniformly bounded in $d$. In other words, we show that there exists a universal constant $\delta > 0$ such that for any $x \in \{-1,1\}^n$ and ``most'' $Z \in \mathbb{R}^{n \times d}$ (under the measure induced by $f(Z|X=\bx)$), there exists $x' \in \{-1,1\}^n$ such that $f(Z|X=x') \ge \delta f(Z|X=x)$. This spread in the conditional density, and consequently the posterior, directly impacts the accuracy of any estimator for the labels, including the MAP estimator. The scaling $n/d = \rho \in (0,1)$ is used in our argument to obtain enough concentration of a certain quadratic form that in turn determines the conditional density (Remark~\ref{rem:conditional-density-energy-form} and Lemma~\ref{eqn:energy-planted}). 
Our label recovery lower bound of Theorem \ref{thm:label_rec} is thus slightly below the parameter recovery lower bound of Theorem \ref{IT_param}, leaving open the regime $n=d-o(d)$. Closing this gap would require an alternative concentration argument, which we leave as an open question. 

Finally, let us relate these lower bounds with those established in the literature. 
Statistical lower bounds for binary Gaussian mixtures have been extensively studied previously, e.g. in \cite{lu2016statistical, giraud2019partial, royer2017adaptive, ndaoud2018sharp,unknown-cov}. In our setting, although the mean and covariance are unknown, their particular relationship, i.e. $\Sigma = I - uu^\top$, is known to the statistician, resulting in a sharp transition at $n=(1-o(1))d$ with explicit constant. We emphasize however that we focus on exact recovery rather than misclassification rate guarantees. Lastly, we operate in a regime where the covariance is singular, as opposed to \cite{unknown-cov} where the covariance is assumed to be invertible, which requires extra technical care.

\section{Proof of Algorithm \ref{alg:lll} correctness}
\subsection{Towards proving Theorem \ref{thm:LLL}: auxiliary lemmas}
\label{sec:lll_main_lemmas}

We present here three auxiliary lemmas for proving Theorem \ref{thm:LLL} and Lemma \ref{lem:min_norm}. The first lemma establishes that given a small (in $\ell_2$-norm) ``approximate'' integer relation between real numbers, one can appropriately truncate each real number to a sufficiently large number of bits, so that the truncated numbers satisfy a small (in $\ell_2$-norm) integer relation between them. This lemma, which is an immediate implication of \cite[Lemma D.6]{song2021cryptographic}, is important for the appropriate application of the LLL algorithm, which needs to receive integer-valued input. Recall that for a real number $x$ we denote by $(x)_N$ its truncation to its first $N$ bits after zero, i.e.\ $(x)_N:=2^{-N} \lfloor 2^N x \rfloor.$

\begin{lemma}[``Rounding'' approximate integer relations {\cite[Lemma D.6]{song2021cryptographic}}]
\label{lem:trunc}
Let $d \in \mathbb{N}$ be a number and let $n \in \mathbb{N}$ be such that $n \leq C_0d$ for some constant $C_0>0$. Moreover, suppose for some constant $C_1 > 0$, a (real-valued) vector $s \in \mathbb{R}^n$ satisfies $\inner{m,s} = 0$ for some $m \in \mathbb{Z}^n$. Then for some sufficiently large constant $C>0$, if $N=\lceil d^4 (\log d)^2 \rceil$, there is an  $m' \in \mathbb{Z}^{n+1}$ which is equal to $m$ in the first $n$ coordinates, satisfies $\|m'\|_2 \leq C d^{\frac{1}{2}}\|m\|_2$, and is an integer relation for the numbers $(s_1)_N,\ldots,(s_n)_N, 2^{-N}.$ 
\end{lemma}

We need the following anticoncentration result. 
\begin{lemma}[Anticoncentration of misaligned integer combinations]
\label{lem:poly}
Assume that $d^c>a>1/d^{c}$ for some $c>0$ constant. Let $u \in S^{d-1}$ be an arbitrary unit vector and let $x_1,\ldots,x_{d+1} \in \mathbb{Z}$ be an arbitrary sequence of integers, which are not all equal to zero. Now for a sequence of integers $t=(t_1,\ldots,t_{d+1}) \in \mathbb{Z}^{d+1}$, we define the (multi-linear) polynomial $P_{t}(z_1,\ldots,z_{d+1})$ in $d(d+1)$ variables by
\begin{align}
\label{eqn:integer-combination-polynomial}
P_t(z_1,\ldots,z_{d+1}) = \det(Z)  t_1+\sum_{i=2}^{d+1} \det(Z_{-i})  t_i\;,
\end{align}
where each $z_1,\ldots,z_{d+1}$ is assumed to have a $d$-dimensional vector form, $Z$ denotes the $d \times d$ matrix with $z_2,\ldots,z_{d+1}$ as its columns, and each $Z_{-i}$ for $i = 2, \ldots, d+1$ denotes the $d \times d$ matrix formed by swapping out the $(i-1)$-th column of $Z$ with $-z_1$.

Suppose $z_i$'s are drawn independently from $\mathcal{N}( (a x_i) u, \Sigma)$ for some $u \in \mathcal{S}^{d-1}$ and $\Sigma \in \mathbb{R}^{d \times d}$ which is $u$-weakly separable per Assumption \ref{assum:u-weak_sep} and eigenvalues $0=\lambda_1<\lambda_2 \leq \lambda_3 \leq \ldots \leq \lambda_{d}$. Then, for any $t \in \mathbb{Z}^{d+1}$ it holds that 
\begin{align}\label{eqn:mean_zero}
\mathbb{E}[P_{t}(z_1,\ldots,z_{d+1})]= 0
\end{align}and 
\begin{align}
\label{eqn:variance-of-integer-combination}
\mathrm{Var}(P_{t}(z_1,\ldots,z_{d+1}))= (d-1)! a^{2d} \left(\prod_{i=2}^d \lambda_i\right)^2 \sum_{1 \leq i < j \leq d+1} (t_ix_j-t_jx_i)^2\;.
\end{align}
Furthermore, for some universal constant $B>0$ the following holds. If $t \neq c x$ for any $c \in \mathbb{R}$, where we denote $x = (x_1,\ldots,x_{d+1})$, then for any $\epsilon>0,$ 
 \begin{align} \label{eq:anti} \mathbb{P}(|P_{t}(z_1,\ldots,z_{d+1})| \leq \epsilon ) \leq B d^B \epsilon^{\frac{1}{d}}.
 \end{align}
\end{lemma}

\begin{proof}
We first describe how \eqref{eq:anti} follows from \eqref{eqn:mean_zero} and \eqref{eqn:variance-of-integer-combination}. First, notice that under the assumption on the integer sequence $t_i, i=1,\ldots,d+1$ not being a multiple of the sequence of integers $x_i, i=1,\ldots,d+1$ it holds that for some $i,j=1,\ldots, d+1, i \not =j$ with $(t_ix_j-t_jx_i)^2 \geq 1$. In particular, using \eqref{eqn:variance-of-integer-combination} we have
\begin{align*}
\mathrm{Var}(P_{t}(z_1,\ldots,z_{d+1})) \geq (d-1)! a^{2d}\left(\prod_{i=2}^d \lambda_i\right)^2.\end{align*} But now notice that from Assumption \ref{assum:u-weak_sep} and $a>d^{-c}$, it holds for some constant $C'>0$ that  $$a^{2d}\left(\prod_{i=2}^d \lambda_i\right)^2 \geq d^{-C'd}.$$Hence, it holds that
\begin{align*}
\mathrm{Var}(P_{t}(z_1,\ldots,z_{d+1}))  \geq d^{-C'd}.\end{align*} Now we employ \cite[Theorem 1.4]{v012a011} (originally proved in \cite{Carbery2001DistributionalAL}) which implies that for some universal constant $B>0,$ since our polynomial is multilinear and has degree $d+1$, it holds for any $\epsilon>0$ that
\begin{align*}  
\mathbb{P}\left(|P_{t}(z_1,\ldots,z_{d+1})| \leq \epsilon \sqrt{\mathrm{Var}(P_{t}(z_1,\ldots,z_{d+1}))} \right) \leq B d \epsilon^{\frac{1}{d}}.
\end{align*}Using our lower bound on the variance we conclude the result.

Now we proceed with the mean and variance calculation. 
As this statement is about the first and second moment of $P_t$ and the determinant operator is invariant up to basis transformations, we may assume without loss of generality that $u=e_1,$ that is, $u$ is equal to the first standard basis vector, and the remaining standard basis vectors are the remaining eigenvectors of $\Sigma$. Recall that $z_i$'s are drawn in an independent fashion from $\mathcal{N}((ax_i)u, \Sigma)$. Hence for a sequence of i.i.d.\ $w_i \sim \mathcal{N}(0,I_{d-1}), i=1,\ldots,d+1$ we may assume from now on that,
\begin{align}
z_i=\begin{bmatrix} a x_i \\ \Lambda w_i \end{bmatrix}
\end{align}for $\Lambda:=\mathrm{diag}(\lambda_2,\ldots,\lambda_d).$

Now let us define the $(d-1) \times (d-1)$ matrix $W_{-j}$ for each $2 \le j \le d+1$ as the matrix formed using $w_2,\ldots,w_{d+1}$  \emph{except} $w_j$ as its column vectors, and define functions $\psi_i : \mathbb{R}^{(d-1)\times(d-1)} \rightarrow \mathbb{R}$ for each $i = 2,\ldots,d+1$ to be the determinant of $W_{-j}$ with the column corresponding to $w_i$ swapped by $-w_1$. For instance, if $2 \le i \not = j \le d+1$, then
\begin{align}
\psi_i(W_{-j}):=\det(w_2,\ldots, w_{i-1}, -w_1,w_{i+1},\ldots,w_{j-1},w_{j+1},\ldots, w_{d+1}).
\end{align}
We abuse notation and also write
\begin{align} 
\psi_1(W_{-j}):=\det(w_2,\ldots,w_{j-1},w_{j+1},\ldots w_{d+1})=\det(W_{-j}).
\end{align} 

As the result is clearly $a$-homogeneous of degree $2d$ we assume in what follows that $a=1$. Now by direct expansion along the first row of the corresponding matrices we have
\begin{align*}
    \det(Z)=\sum_{j=2}^{d+1} (-1)^{j}x_{j}|\det(\Lambda)|\psi_1(W_{-j})\;,
\end{align*}
and for each $i \geq 2$,
\begin{align*}
\det(Z_{-i}) &:=(-1)^{i+1}x_{1}|\det(\Lambda)|\psi_1(W_{-i})+\sum_{j=2, j \not = i}^{d+1} (-1)^{j}x_{j} |\det(\Lambda)|\psi_i(W_{-j})\;.
\end{align*}
Since $d>1$ and $w_i$ are i.i.d.\ $\mathcal{N}(0,I_d)$ we can immediately conclude that for all $i \geq 1, j \geq 2, i \not =j$, $$\mathbb{E}[\psi_i(W_{-j})]=0.$$ Hence,  \begin{align*}\mathbb{E}[P_{t}(z_1,\ldots,z_{d+1})]=t_1 \mathbb{E}[\det(Z)]+\sum_{i=2}^{d+1} t_i \mathbb{E}[\det(Z_{-i})]=0.\end{align*}

Now we calculate the second moment of the polynomial. In what follows, we slightly abuse notation and denote  $Z_{-1}:=Z$ for notational convenience. First, again by direct expansion of the determinant and the fact that $w_i$'s for $i=1,\ldots,d+1$ have i.i.d.\ standard Gaussian entries it holds by direct inspection that for all $i, j \in [d+1]$ with $i \neq j$,
\begin{align}\label{psi1}
    \mathbb{E}[\psi_i(W_{-j})^2]&=(d-1)!|\det(\Lambda)|^2\;,
\end{align} 
and unless $\{i,j\} = \{k,\ell\}$, it holds that
\begin{align}\label{psi2}
 \mathbb{E}[\psi_i(W_{-j})\psi_k(W_{-\ell})]=0.
\end{align}

We now calculate for $2 \leq i \not = j \leq d+1$ the term  $\mathbb{E}[\psi_i(W_{-j})\psi_j(W_{-i})].$ We assume without loss of generality that $i < j$. Notice that for $\Pi_c \in \{0,1\}^{d-1 \times d-1}$, the permutation matrix corresponding to the cycle-permutation $c:=(i-1, i,\ldots,j,j-1) \in \mathrm{Sym}([d-1])$, the matrix 
\begin{align*}
    (w_2, \ldots w_{i-1},-w_1, w_{i+1},\ldots, w_{j-1},w_{j+1},\ldots, w_{d+1})\;,
\end{align*}
equals
\begin{align*}
    \Pi_c (w_2, \ldots w_{i-1},w_{i+1},\ldots, w_{j-1}, -w_1,w_{j+1},\ldots, w_{d+1})\;.
\end{align*}
Hence, 
\begin{align*}
    \psi_i(W_{-j})\psi_j(W_{-i})=\det(\Pi_c) \psi^2_i(W_{-j})=(-1)^{\sgn(c)}\psi^2_i(W_{-j})=(-1)^{i-j+1}\psi^2_i(W_{-j})\;.
\end{align*}
In particular,
\begin{align}
\label{psi3}
\mathbb{E}[\psi_i(W_{-j})\psi_j(W_{-i})]=(-1)^{i-j+1}(d-1)!|\det(\Lambda)|^2\;.
\end{align}

Now using \eqref{psi1}, \eqref{psi2}, we have for each $1 \leq i \leq d+1$,
\begin{align}
    \mathbb{E}[\det(Z_{-i})^2]=(d-1)! \sum_{j = 1, j \neq i}^{d+1} x_j^2|\det(\Lambda)|^2\;,
\end{align} 
and using \eqref{psi1}, \eqref{psi2}, and \eqref{psi3} we have for all $i \neq j$ that
\begin{align}
    \mathbb{E}[\det(Z_{-i})\det(Z_{-j})]=-(d-1)! x_ix_j|\det(\Lambda)|^2\;.
\end{align}

Hence, it holds that

\begin{align*}
\mathbb{E}[P_{t}(z_1,\ldots,z_{d+1})^2]
&=|\det(\Lambda)|^2\sum_{i,j=1}^{d+1} t_it_j \mathbb{E}[\det(Z_{-i})\det(Z_{-j})]\\
&=|\det(\Lambda)|^2\sum_{i=1}^{d+1} t_i^2 \mathbb{E}[\det(Z_{-i})^2]+\sum_{i,j=1,  i \not =j}^{d+1} t_it_j \mathbb{E}[\det(Z_{-i})\det(Z_{-j})]\\
&=(d-1)!|\det(\Lambda)|^2 \left(  \sum_{i, j=1, i \neq j }^{d+1}t_i^2x_j^2 - \sum_{i,j=1, i \neq j}^{d+1} t_it_j x_ix_j\right)\\
&=(d-1)!|\det(\Lambda)|^2\left( \sum_{1\leq i<j \leq d+1}(t_ix_j-t_jx_i)^2 \right)\;.
\end{align*}

\end{proof}

The following lemma establishes multiple structural properties of the $d+1$ samples.

\begin{lemma}
\label{lem:bounds}
Let $u \in S^{d-1}$ be an arbitrary unit vector and let $x_i \in \mathbb{Z}\cap[-2^d,2^d]$ for $i=1,\ldots,d+1$ be arbitrary integers which are not all equal to zero. Let also spacing $a$ with $d^{-c}<a<d^c$ for some $c>0$ and $\Sigma$ which is $u$-weakly separable per Assumption \ref{assum:u-weak_sep}. We observe $d+1$ samples of the form $z_i$, where for each $i = 1, \ldots, d+1$, $z_i$ is an independent sample from $\mathcal{N}((ax_i) u, \Sigma)$. We denote by $Z \in \mathbb{R}^{d \times d}$ the (random) matrix with columns given by the $d$ vectors $z_2,\ldots,z_{d+1}$.  The following properties hold.
\begin{itemize}

   \item[(1)] The matrix $Z$ is invertible almost surely. \item[(2)] With probability $1-\exp(-\Omega(d))$ over the $z_i$'s,  
        $$\|Z^{-1}z_1\|_{\infty} =O( 2^{2d^2}).$$ 
    \item[(3)]  With probability $1-\exp(-\Omega(d))$ over the $z_i$'s,
    $$0<|\mathrm{det}(Z)|=O(2^{d^2}).$$ 
\end{itemize}
\end{lemma}

\begin{proof} 
For the fact that $Z$ is invertible, consider its determinant, that is, the random variable $\det(Z)$. We claim that $\det(Z) \not =0 $ almost surely. Note that to prove this, by invariance of the determinant to the change of basis, we may assume without loss of generality that $u=e_1$, that is, $u$ is the first standard basis vector, and the remaining standard basis vectors are the remaining eigenvectors of $\Sigma$. Under this assumption, for each $i = 1,\ldots, d+1$, we can write using Assumption \ref{assum:u-weak_sep}
\begin{align*}
    z_i = \begin{bmatrix}a x_i \\ \Lambda w_i \end{bmatrix} \;,
\end{align*}
where $\Lambda=\mathrm{diag}(\lambda_2,\ldots,\lambda_d)$ and $w_i$'s are i.i.d.\ samples from $\mathcal{N}(0,I_{d-1})$. In other words, the first row of $Z$ consists of $ax_2,\ldots,ax_{d+1}$, and the rest are coordinates of $\Lambda w_i$, where each $w_i$ is a vector with i.i.d.\ standard Gaussian entries. Now the result follows from the fact that since not all $x_i$ are equal to zero and also none of the $\lambda_i$'s are zero from Assumption \ref{assum:u-weak_sep}, the determinant $\det(Z)$ with fixed $x_2,\ldots,x_{d+1}$ is a non-zero polynomial of the entries of $w_2,\ldots,w_{d+1}$. As all entries of $w_i$ are distributed as i.i.d.\ standard Gaussians, the random polynomial $\det(Z)$ is almost surely non-zero~\cite{caron2005zero}.

For the second part, notice that by Cramer's rule for $i=1,\ldots,d-1$, the $i$-th coordinate of $Z^{-1}z_1$ equals the quantity $\lambda_{i+1}(Z):=\det(z_2,\ldots,z_{i},-z_1, z_{i+1},\ldots, z_{d+1})/\det(Z)$ almost surely. Hence, again by the rotational invariance property of the determinant operator, we may assume that $u=e_1$ and the remaining standard basis vectors are the remaining eigenvectors of $\Sigma$. Let $q^{(i)} \in \mathbb{Z}^{d+1}$ be an integer-valued vector such that $q^{(i)}_j = 1$ if $i=j$ and $q^{(i)}_j = 0$ otherwise. Now using the notation of Lemma \ref{lem:poly} we have that $P_{q^{(i)}}(z_1,\ldots,z_{d+1})=\det(Z_{-i})$. By applying the anticoncentration result from Lemma~\ref{lem:poly} for the polynomial $P_{q^{(1)}}(z_1,\ldots,z_{d+1})$ and $\epsilon=2^{-d^2}$ we conclude that
\begin{align}\label{detIn}
    |\det(Z)|=|P_{q^{(1)}}(z_1,\ldots,z_{d+1})| \geq 2^{-d^2}
\end{align} with probability $1-\exp(-\Omega(d))$. Furthermore, for all $i=1,\ldots,d+1$ it holds that 
\begin{align*}
\mathbb{E}[P_{q^{(i)}}(z_1,\ldots,z_{d+1})^2]=\mathrm{Var}(P_{q^{(i)}}(z_1,\ldots,z_{d+1})) = a^{2d} d!\|x\|_2 \le a^{2d} |\det(\Lambda)|^2 2^{10d \log d}\|x\|^2_2\;,
\end{align*}where $x:=(x_1,\ldots,x_{d+1})^{\top}$ where $\Lambda=\mathrm{diag}(\lambda_2,\ldots,\lambda_{d})$ and $\lambda_i, i>1$ are the non-zero eigenvalues of $\Sigma$ per Assumption \ref{assum:u-weak_sep}. Hence, by Markov's inequality, the fact that $a<d^c,$ the Assumption \ref{assum:u-weak_sep} and a union bound over $i$, we have for all $i=1,\ldots,d+1$ that
\begin{align}\label{det2In}
  |P_{q^{(i)}}(z_1,\ldots,z_{d+1})| \le 2^{d^2/2}\|x\|^2_2
\end{align} with probability $1-\exp(-\Omega(d)).$

Combining Eq.\eqref{detIn} and Eq.\eqref{det2In}, we conclude that for all $i=2,\ldots,d$, $$|\lambda_i(Z)|=|P_{q^{(i)}}(z_1,\ldots,z_{d+1})/P_{q^{(1)}}(z_1,\ldots,z_{d+1})| \leq 2^{3d^2/2}\|x\|^2_2$$ with probability $1-\exp(-\Omega(d))$. Since $\|x\|^2_2 =O(2^{2d})$ we have $\|Z^{-1}z_1\|_{\infty} \leq 2^{3d^2/2}\|x\|^2_2\leq 2^{2d^2}$ with probability $1-\exp(-\Omega(d))$. This concludes the proof of the second part. 

Finally, Eq.\eqref{det2In} for $i=1$ and the fact $\|x\|^2_2 =O(2^{2d})$ imply
\begin{align}
  |\det(Z)|=|P_{q^{(1)}}(z_1,\ldots,z_{d+1})| \le 2^{d^2}
\end{align} with probability $1-\exp(-\Omega(d))$. 
This concludes the proof of the third part.
\end{proof}

\subsection{Proof of Theorem \ref{thm:LLL}}
\label{sec:appD3}

We now proceed with the proof of the Theorem \ref{thm:LLL} using the lemmas from the previous sections.

{
\def\thetheorem{\ref{thm:LLL}}
\begin{theorem}[Restated]
Algorithm \ref{alg:lll}, given as input independent samples $(z_i)_{i=1,\ldots,d+1}$ from Model \ref{mod:general_model} with hidden direction $u$, covariance $\Sigma,$ and true labels $\{x_i\}_{i =1,\ldots,d+1}$ satisfies the following with probability $1-\exp(-\Omega(d))$: there exists $\eps \in \{-1,1\}$ such that the algorithm's outputs $\{\hat{x}_i\}_{i=1,\ldots,d+1}$ and $\hat{u} \in S^{d-1}$ satisfy
\begin{align*}
    \hat{x}_i&= \epsilon x_i\;\, \text{ for } i=1,\ldots,d+1\\
   \text{ and  } \hat{u} & =\epsilon u\;.
\end{align*} 
Moreover, Algorithm \ref{alg:lll} terminates in $\poly(d)$ steps.
\end{theorem}
}

\begin{proof}
We start with noticing that for an algorithm to recover $u, x_i$ up a to global sign flip it suffices to recover the values of $\{x_i\}_{i=2,\ldots,d+1}$ up to a global non-zero constant multiple. Indeed, since we already know the value of $z_i$'s, if we learn the $x_i$'s up to a constant, call it $C>0$, then we can solve the linear system of $d$ (independent) equations and with $d$ unknowns given by $\inner{z_i,v}=C a x_i=C\inner{z_i,u}, i=2,\dots,d+1.$ Since by Lemma \ref{lem:bounds} the matrix $Z$, also formed in Algorithm \ref{alg:lll}, which is the $d \times d$ matrix with $z_2,\ldots,z_{d+1}$ as its column vectors, is invertible almost surely, one can indeed solve this linear system to recover $v=Cu$, that is the same constant $C$ times $u$. Since $u$ is assumed to be unit norm one can then recover the quantity $|C|=\|v\|_2,$ which is the absolute value of the unknown constant. Hence one can output for some $\epsilon=C/|C| \in \{-1,1\}$ the estimated vector $C u/|C|=\epsilon u$ and the estimated labels  $Cx_i/|C|=\epsilon x_i, i=1,\ldots,d+1$ which are indeed the hidden direction $u$ and the true labels $x_i, i=1,\ldots,d+1$  up to a global sign flip. 

Now our proposed Algorithm \ref{alg:lll} follows exactly this path: it first recovers a non-zero constant multiple of the $x_i$'s (this is the values of the vector $t_1$ output by the LLL step) with probability $1-\exp(-\Omega(d)).$  Then it uses the simple procedure described above to output both the labels $x_i, i=1,\ldots,d+1$ and $u$ up to a global constant multiple. This second part comprises exactly the last steps of the algorithm after the LLL step. The main procedure of our algorithm therefore is to use an appropriate application of LLL to learn the exact values of $x_i$ up to a global sign flip. We now analyze the success of the LLL step to recover a global constant multiple of the $x_i$'s with probability $1-\exp(-\Omega(d)).$

Now the algorithm does not terminate in the second step exactly because of the almost sure invertibility of the matrix $Z$, per Lemma \ref{lem:bounds}. Let us now analyze the (random) lattice $L=L(B)$ generated by the basis $B$, which is constructed in the next step of Algorithm \ref{alg:lll}.

First, observe that the real numbers $\{\lambda_i\}_{i=1,2,\ldots,d+1}$ used in the top row of the lattice basis $B$, satisfy by definition
\begin{align*}
\sum_{i=1}^{d+1} \lambda_i z_i = 0\;.
\end{align*} 
Hence, we conclude that since $\inner{z_i,u}=ax_i$ for the unknown direction $u \in S^{d-1}$ and spacing $a>0$, it holds that
\begin{align}\label{IR0}
\sum_{i=1}^{d+1} \lambda_i a x_i=\sum_{i=1}^{d+1} \lambda_i  \langle u, z_i \rangle = \langle u, \sum_{i=1}^{d+1} \lambda_i  z_i \rangle = 0\;
\end{align} and therefore
\begin{align}\label{IR}
\sum_{i=1}^{d+1} \lambda_i  x_i=  0\;.
\end{align} 

We now show an upper bound on the shortest vector length of $L$, which we denote by $\mu(L)$. More precisely, we show that
\begin{align*}
    \mu(L) = O(d2^{d})\;.
\end{align*} To this end, define a real-valued vector $s \in \mathbb{R}^{d+1}$ with $s_i=\lambda_i$ for $i=1,\ldots,d+1$, and also an integer-valued vector $m \in \mathbb{Z}^{d+1}$ with $m_i=x_i$ for $i=1,\ldots,d+1$. Then, the integer relation \eqref{IR} implies that $\inner{s,m}=0$.  Since $|x_i| \leq 2^d$ for all $i=1,\ldots,d+1$ it also holds almost surely that $\|m\|_2 =\|x\|_2 \leq \sqrt{d}2^d$. By Lemma \ref{lem:trunc}, for the bit-precision $N$ chosen by Algorithm \ref{alg:lll}, there exists an integer $m'_{d+2} \in \mathbb{Z}$ such that $m'=(m,m'_{d+2}) \in \mathbb{Z}^{d+2}$ satisfies $\|m'\|_2=O(d2^d)$ and is an integer relation for $(\lambda_1 )_N,\ldots,(\lambda_{d+1})_N, 2^{-N}$.

Now define $b \in (2^{-N}\mathbb{Z})^{d+2}$ given by $b_i=(\lambda_i)_N$ for $i=1,\ldots,d+1$, and $b_{d+2}=2^{-N}.$ Notice that $b_1=(1)_N=1$ and furthermore that the $\tilde{v}$ defined by the algorithm satisfies $\tilde{v}=(b_2,\ldots,b_{d+2}).$ On top of this, we have that the $m'$ defined in previous paragraph is an integer relation for $b$ with $\|m'\|_2=O(d2^d)$. Hence, $Bm'=(0,m')^{\top}.$ It follows that $\mu(L) = O(d2^d)$ with probability $1-\exp(-\Omega(d))$, since $\mu(L) \le \|B m'\|_2 = O(d2^d)$.

Recall from Theorem~\ref{LLL_thm_original} that the LLL algorithm is guaranteed to return a lattice vector of $\ell_2$-norm smaller than $2^{\frac{d+2}{2}} \mu(L)$. Now we employ Lemma~\ref{lem:min_norm} which combined with the fact that $2^{\frac{d+2}{2}} \mu(L) \leq 2^{2d}$ for sufficiently large $d$ almost surely, allows us to conclude that the LLL algorithm returns a non-zero lattice vector $B(t_1,t_2)^\top$, where $t_1 \in \mathbb{Z}^{d+1}$ and $t_2 \in \mathbb{Z}$, such that $t_1$ is an integer multiple of $x=(x_1,\ldots,x_{d+1})$ with probability $1-\exp(-\Omega(d))$. Hence, using $t_1$ the algorithm recovers a global non-zero constant multiple of the $x_i$'s for $i=1,\ldots,d+1$ with probability $1-\exp(-\Omega(d))$.

For the termination time, it suffices to establish that the step using the LLL basis reduction algorithm can be performed in $\poly(d)$ time. To ensure $\poly(d)$ time for the LLL step, it suffices to show that the entries of the lattice basis $B$ are not too large with probability $1-\exp(-\Omega(d))$. More precisely, the running time of LLL depends on the logarithm of the largest entry in $B$ by Theorem \ref{LLL_thm_original}. Clearly, $N$ and $\log M$ are polynomial in $d$. Finally, direct inspection and Lemma \ref{lem:bounds} implies that the quantity $\log \|\lambda\|_{\infty}$, where $\lambda=(\lambda_1,\ldots,\lambda_{d+1})^\top$ is as defined in Algorithm \ref{alg:lll}, is polynomially bounded with probability $1-\exp(-\Omega(d))$. This establishes the $\poly(d)$ running time of the LLL step. 
\end{proof}

\subsection{Proof of Lemma \ref{lem:min_norm}}
\label{sec:proof_lem_LLL}

We focus this section on proving the key technical Lemma \ref{lem:min_norm}. As mentioned above, the proof of the lemma is quite involved, and, potentially interestingly, it requires the use of anticoncentration properties of the coefficients $\lambda_i$, which are rational functions of the coordinates of $x_i$, as discussed in Lemma \ref{lem:poly}.

{
\def\thetheorem{\ref{lem:min_norm}}
\begin{lemma}[Restated]
Let $ d \in \mathbb{N}, $ $a \in [d^{-c},d^c]$ for some $c>0$ and $N=\lceil d^4(\log d)^2 \rceil$. Let $u \in S^{d-1}$ be an arbitrary unit vector, $\Sigma \in \mathbb{R}^{d \times d}$ an arbitrary $u$-weakly separable matrix and let $x_i \in \mathbb{Z}\cap [-2^d,2^d]$ for $i=1,\ldots,d+1$ be arbitrary but not all zero. Moreover, let $\{z_i\}_{i=1,\ldots,d+1}$ be independent samples from $N((ax_i)u,\Sigma)$, and let $B$ be the matrix constructed in Algorithm \ref{alg:lll} using $\{z_i\}_{i=1,\ldots,d+1}$ as input and $N$-bit precision. Then, with probability $1-\exp(-\Omega(d))$ over the samples, for any $t = (t_1, t_2) \in \mathbb{Z}^{d+1} \times \mathbb{Z}$ such that $t_1$ is not an integer multiple of $x=(x_1,\ldots,x_{d+1})$, the following holds:
\begin{align*}
    \|Bt\|_2 > 2^{2d}\;.
\end{align*}
\end{lemma}
}

\begin{proof}[Proof of Lemma \ref{lem:min_norm}]
Let $t = (t_1,t_2) \in \mathbb{Z}^{d+1} \times \mathbb{Z}$ be arbitrary non-zero integer coefficients. Our proof consists of characterizing integer coefficients $t$ for which the corresponding lattice vector $Bt$ is ``short'', that is, \begin{align}\label{short}\|Bt\|_2 \le 2^{2d}.
\end{align} In what follows, by a \textit{short lattice vector} we refer to the condition \eqref{short}.

We first show that with probability $1-\exp(-\Omega(d))$, lattice vectors can only be short for integer coefficients contained in some bounded rectangle $\sR \subset \mathbb{Z}^{d+2}$, which we define below (see Eq.\eqref{eqn:bounded-rectangle}). Then, we apply our anticoncentration lemma (Lemma~\ref{lem:poly}) and a union bound over a subset of $\sR$ to conlcude that with probability $1-\exp(-\Omega(d))$, the only short lattice vectors are ones whose integer coefficients satisfy $t_1 = c x$ for some $c \in \mathbb{Z}$.

To this end, we first observe that entries of the first row of $B$ are elements of $M \mathbb{Z}$, as by direct inspection $(Bt)_1=M(\sum_{i=1}^{d+1}\left(2^N(\lambda_i)_N\right) (t_1)_i +t_2)$. It follows that if $t$ is not an integer relation for the numbers $(\lambda_1)_N,\ldots,(\lambda_{d+1})_N, 2^{-N}$, then $\|Bt\|_2 \geq M=2^{2d}$. Hence, it suffices to restrict our attention to $t$'s which are integer relations, that is,
\begin{align*}
\sum_{i=1}^{d+1}(\lambda_i)_N (t_1)_i +t_2 2^{-N}=0\;.
\end{align*}
Note that it cannot be the case that $t_1=0$ since this implies, by the integer relation above, $t_2=0$, and therefore the pair $t=(t_1,t_2)$ are zero, a contradiction. Hence, from now on we restrict ourselves only to the case where $t_1 \neq 0$.

Let us denote by $t'$ the vector $t$ without the first coordinate $(t_1)_1$, i.e., $t' = ((t_1)_2,\ldots,(t_1)_{d+1}, t_2)$. Our second observation is that $\|Bt\|_2 \ge \|t'\|_\infty$ because of the use of the submatrix $I_{d+1}$ in the definition of $B$. This implies that any short lattice vector $Bt$ must satisfy $\|t'\|_\infty \le 2^{2d}$. Moreover, since $t$ is an integer relation and $\lambda_1=1$, we have
\begin{align}
\label{eqn:bounded-rectangle0}
    |(t_1)_1| = \left|\sum_{i=2}^{d+1} (\lambda_i)_N (t_1)_i + t_2 2^{-N}\right| \le \|t'\|_\infty \left(\|\lambda\|_1 + 2^{-N} \right)\;.
\end{align}

Now in the notation of Lemma \ref{lem:bounds} we have $\lambda=-Z^{-1}z_1.$ Hence using Lemma \ref{lem:bounds} and the elementary fact that $\|\lambda\|_1 \le (d+1)\|\lambda\|_\infty$, it holds with probability $1-\exp(-\Omega(d))$ that $\|\lambda\|_1 = O(2^{2d^2})$. It follows that, for sufficiently large $d$, any short lattice vector $Bt$ must satisfy $|(t_1)_1| \le 2^{3d^2}$ with probability $1-\exp(-\Omega(d))$. Hence, with probability $1-\exp(-\Omega(d))$, every short vector $Bt$ in the random lattice $L=L(B)$ has its integer coefficients $t$ contained in $\sR$, which is defined as
\begin{align}
\label{eqn:bounded-rectangle}
    \sR = \{(a,b) \in \mathbb{Z} \times \mathbb{Z}^{d+1} \,:\, |a| \le 2^{3d^2}, \|b\|_\infty \le 2^{2d}\}\;.    
\end{align}
From $\sR$, we also define $\sR_1 \subset \mathbb{Z}^{d+1}$ such that $\sR_1 = \{t_1 \in \mathbb{Z}^{d+1} \,:\, t=(t_1,t_2) \in \sR \}$.

We now show using a union bound over $t \in \sR$ that with probability $1-\exp(-\Omega(d))$, the only short lattice vectors in $L$ are ones whose integer coefficients $t=(t_1,t_2)$ satisfy $t_1 = cx$ for some $c \in \mathbb{Z}$. First, observe that since $|t_2| \leq 2^{2d}$, the following inequality holds if $t$ is an integer relation:
\begin{align*} 
\left|\sum_{i=1}^{d+1}(\lambda_i)_N (t_1)_i \right| \leq 2^{2d}2^{-N}\;.
\end{align*}

Consider $\sT$ the set of all  $t_1 \in \mathbb{Z}^{d+1} \setminus \bigcup_{c \in \mathbb{R}} \{c (x_1,\ldots,x_{d+1})^{\top}\}.$ To prove our result it suffices to prove that
\begin{align*}
\mathbb{P}\left(\bigcup_{t_1 \in \sT \cap \sR_1}\left\{\left|\sum_{i=1}^{d+1}(\lambda_i)_N (t_1)_i \right|\leq 2^{2d}/2^N\right\}\right) \leq \exp(-\Omega(d))
\end{align*}
for which, since for any $x$ it holds $|x-(x)_N| \leq 2^{-N}$ and $\|t\|_1 = |(t_1)_1| + \|t'\|_\infty \leq 2^{4d^2} $ for sufficiently large $d$, it suffices to prove that for large $d$,
\begin{align*}
    \mathbb{P}\left(\bigcup_{t_1 \in \sT \cap \sR_1}\left\{\left|\sum_{i=1}^{d+1}\lambda_i (t_1)_i \right| \leq 2^{5d^2}/2^N\right\}\right) \leq \exp(-\Omega(d))\;.
\end{align*}

Using the polynomial notation of Lemma \ref{lem:poly} (specifically, Eq.\eqref{eqn:integer-combination-polynomial}), as well as the fact that by Cramer's rule $\lambda_i$ are rational functions of the coordinates of $z_i$ satisfying $\lambda_i \mathrm{det}(z_2,\ldots,z_{d+1})=\mathrm{det}(\ldots, z_{i-1},-z_1,z_{i+1},\ldots)$, it suffices to show
\begin{align*} 
\mathbb{P}\left(\bigcup_{t_1 \in \sT \cap \sR_1}\{|P_{t_1}(z_1,\ldots,z_{d+1})| \leq |\mathrm{det}(z_2,\ldots,z_{d+1})|2^{5d^2}/2^N\}\right) \leq \exp(-\Omega(d))\;.
\end{align*}

By Lemma \ref{lem:bounds}, with probability $1-\exp(-\Omega(d))$ there exists some constant $D>0$ such that $\det(z_2,\ldots,z_{d+1}) \le D 2^{2d^2}$. Hence, it suffices to show
\begin{align*}
\mathbb{P}\left(\bigcup_{t_1 \in \sT \cap \sR_1} \{|P_{t_1}(z_1,\ldots,z_{d+1})|\leq  D 2^{7d^2 }/2^N\}\right) \leq \exp(-\Omega(d))\;.
\end{align*}
Now since $N=\omega(d^2\log d)$, it suffices to show, for sufficiently large $d$,
\begin{align*}
\mathbb{P}\left(\bigcup_{t_1 \in \sT \cap \sR_1} \{|P_{t_1}(z_1,\ldots,z_{d+1})|\leq  2^{-\frac{N}{2}}\}\right) \leq \exp(-\Omega(d))\;.
\end{align*}
By a union bound, it suffices to show
\begin{align}
\label{eq:sum}
\sum_{t_1 \in \sT \cap \sR_1}\mathbb{P}\left(|P_{t}(z_1,\ldots,z_{d+1})| \leq 2^{-\frac{N}{2}}\right) \leq 2^{-\Omega(d)}.
\end{align}

Now the number of integer points $t_1$ with $\ell_\infty$ norm at most $2^{3d^2}$ is at most $2^{3d^2(d+1)}$, since there are at most $2^{3d^2}$ choices per coordinate.
Furthermore, using the anticoncentration inequality \eqref{eq:anti} of Lemma \ref{lem:poly}, we have for any $t_1 \in \sT$ that for some universal constant $B>0,$
\begin{align*}
\mathbb{P}\left(|P_{t_1}(z_1,\ldots,z_{d+1})| \leq 2^{-\frac{N}{2}}\right) \leq Bd2^{-\frac{N}{2d}}\;.
\end{align*}
Using the above to upper bound the left hand side of~\eqref{eq:sum}, we see that the sum is at most
\begin{align*}
B d 2^{3d^2(d+1)} 2^{-\frac{N}{2d}} =\exp( O(d^3) - \Omega(N/d))=\exp(-\Omega(d))\;,
\end{align*}
where we used that $N/d=\Omega(d^3 \log d )$. This completes the proof.
\end{proof}
\section{Proofs of information-theoretic lower bounds}
We now provide the proofs of Theorem~\ref{IT_param} and Theorem~\ref{thm:label_rec}. As mentioned in Section~\ref{sec:it-lower-bound}, our lower bounds show that the sample complexity $n=d+1$ for our LLL-based algorithm is indeed optimal. We first provide the proof of Theorem~\ref{IT_param}, which establishes that even when we have access to the true signs $x_i \in \{-1,1\} $ for $i=1,\ldots,n$, we cannot exactly recover the true hidden direction $\bu \in \sS^{d-1}$ with probability larger than $1/2$ if $n \le d-1$. Then, we prove Theorem~\ref{thm:label_rec} which shows that for exact recovery of the labels, $n=(1-o(1))d$ samples are required.

\subsection{Proof of Theorem~\ref{IT_param}}
{\def\thetheorem{\ref{IT_param}}
\begin{theorem}[Restated]
Let arbitrary $\bx \in \{-1,1\}^n$ be fixed and known to the statistician. Moreover, let $\bu \in \mathcal{S}^{d-1}$ be a uniformly random unit vector. For each $i=1,\ldots, d-1$ let $\bz_i$ be a sample generated independently from $\mathcal{N}(x_i \bu,I-\bu\bu^{\top}).$ Then it is information-theoretically impossible to construct an estimate $\hat{\bu}=\hat{\bu}(\{(x_i,\bz_i)\}_{i=1}^n) \in \mathcal{S}^{d-1}$ which satisfies $\hat{\bu}= \bu$ with probability larger than $1/2.$
\end{theorem}
}

\begin{proof}
We establish the result by proving that the posterior is a uniform distribution on some finite subset of $\mathcal{S}^{d-1}$ of cardinality exactly equal to $2$ almost surely. Notice that if we establish this, our impossibility is implied as follows: the optimal estimator in minimizing probability of error for exact recovery is the MAP estimator (see e.g., \cite[Lemma H.4]{song2021cryptographic}). The MAP estimator outputs the $\bv \in \mathcal{S}^{d-1}$ with the maximum posterior mass. Since the posterior is a uniform distribution between two points, the probability of the MAP estimator (and therefore any estimator) in recovering exactly $\bu$ is at most $1/2$. 

To this end, we calculate the posterior mass assigned to any arbitrary fixed vector $\bv \in \sS^{d-1}$, given the samples $\{\bz_i\}_{i=1}^{d-1}$. Let us first complete $\bv$ to an arbitrary ordered orthonormal basis of $\mathbb{R}^d,$ say $\bv=\bq_1,\bq_2,\ldots,\bq_d.$  Then, since the labels $x_i$'s are known, if $\bu$, the true hidden direction, were equal to $\bv$, the samples $\bz_i$ for $i=1,\ldots,d-1$ would admit the basis representation 
\begin{align*}
\bz_i=\sum_{j=1}^{d}c_{i,j} \bq_j=x_i \bv+ \sum_{j=2}^d c_{i,j} \bq_j\;,
\end{align*}
where $c_{i,j}$ for $i \in [d-1]$, and $j \in [d]\setminus\{1\}$ are i.i.d.\ standard Gaussian random variables. Hence, given $\bz_i$ the posterior mass assigned to $\bv$ is zero if $\inner{\bz_i,\bv} \neq x_i$ and otherwise, it has mass proportional to
\begin{align*}
\frac{1}{(2\pi)^{(d-1)/2}}\exp\left(-\sum_{j=2}^{d}\inner{\bq_j,\bz_i}^2/2\right)=\frac{1}{(2\pi)^{(d-1)/2}}\exp(x_i^2-\|\bz_i\|^2_2)=\frac{1}{(2\pi)^{(d-1)/2}}\exp(1-\|\bz_i\|^2_2)\;.
\end{align*}
Notice importantly that the computed quantity is constant with respect to the direction $\bv$ due to the hard constraint $\langle \bz_i, \bv \rangle = x_i$ and the fact that $x_i^2 = 1$ for all $i \in [d-1]$. Hence, the posterior mass at $\bv \in \mathcal{S}^{d-1}$ given the single sample $\bz_i$ is proportional to $\one[\bv \in \mathcal{S}^{d-1} \wedge \inner{\bv,\bz_i}=x_i]$. Since the $\bz_i$'s are generated independently, the posterior measure given $\{z_i\}_{i=1}^{d-1}$ is proportional to $\one[\bv \in \mathcal{S}^{d-1} \wedge Z\bv=\bx]$, where $Z$ is a matrix with $\bz_i^\top$'s as its rows. In what follows, let us call then 
\begin{align*}
    S :=\{\bv \in \mathcal{S}^{d-1} \mid Z\bv=\bx\}\;.
\end{align*}

To upper bound the success probability of any estimator by $1/2$, it suffices to show $|S|=2$ almost surely. We first prove that $|S| \leq 2$ almost surely. Recall that $Z \in \mathbb{R}^{(d-1) \times d}$ is a matrix with rows $\bz_i^\top,i=1,\ldots,d-1.$ Notice that to prove $|S| \leq 2$ almost surely, it suffices to show that the $\mathrm{Kernel}(Z)$ is a one-dimensional linear subspace (i.e., consists of points on a line passing through the origin) almost surely. Indeed, since the true direction $\bu \in \mathcal{S}^{d-1}$ satisfies $\inner{\bu,\bz_i}=x_i$ we have
\begin{align*}
S=\mathcal{S}^{d-1} \cap (\bu + \mathrm{Kernel}(Z))\;.
\end{align*}

Any line can intersect the sphere in at most two points. Hence, if $\mathrm{Kernel}(Z)$ is one-dimensional almost surely, then $|S| \leq 2$ almost surely. We now show that the $\mathrm{Kernel}(Z)$ is one-dimensional almost surely. By invariance of the kernel of $Z$ to the change of column basis, we may assume without loss of generality that $\bu=\be_1$, that is, $\bu$ is the first standard basis vector, and the remaining orthonormal basis vectors of $\mathbb{R}^d$ are just the standard basis vectors $\be_2,\ldots,\be_d$. Under this assumption, we can write for each $i = 1,\ldots, d+1$,
\begin{align*}
    \bz_i = \begin{bmatrix}a x_i \\  \bw_i \end{bmatrix} \;,
\end{align*}
where the $\bw_i$'s are i.i.d.\ samples from $\mathcal{N}(0,I_{d-1})$.

In other words, the first column of $Z \in \mathbb{R}^{(d-1) \times d}$ consists of $ax_1,\ldots,ax_{d-1}$, and the remaining columns consist of coordinates of $\bw_i \in \mathbb{R}^{d-1}$. Now we proceed by establishing that the last $d-1$ columns of $Z$ in this basis are linearly independent almost surely, which suffices to establish that the kernel is one-dimensional. Let $W \in \mathbb{R}^{(d-1)\times (d-1)}$ be the submatrix of $Z$ consisting of the last $d-1$ columns of $Z$. Notice that the entries of $W$ are i.i.d.\ Gaussian. Now using folklore results (e.g., \cite{caron2005zero}), we have that the determinant of an i.i.d.\ Gaussian matrix is non-zero almost surely. Hence, we conclude that indeed $\det(W) \not = 0$ almost surely. This implies that the column rank of $Z \in \mathbb{R}^{(d-1) \times d}$ is equal to $d-1$ almost surely, and thus $\mathrm{Kernel}(Z)$ is one-dimensional.

We now prove that $|S|\geq 2$ almost surely. First notice that by our assumption on the data generating process, the set $S$ contains at least one unit vector, namely $\bu$, which is drawn uniformly at random from $\mathcal{S}^{d-1}$. Hence, $|S| \geq 1$. To show that $S\setminus\{\bu\}$ is non-empty, we claim there exists $\by \in \mathbb{R}^{d}$ such that $\by \in \mathrm{Kernel}(Z)$ and $\inner{\by,\bu}<0,$ almost surely. Suppose not. Then, by the hyperplane separation theorem $\bu \in \mathrm{span}(\bz_1,\ldots,\bz_{d-1})$. That is, there exist scalars $\lambda_1,\ldots,\lambda_{d-1} \in \mathbb{R}$ such that
\begin{align}
\label{eqsNormal0}
    \bu=\sum_{i=1}^{d-1}\lambda_i \bz_i\;.
\end{align}

Let $\bu_2,\ldots,\bu_{d}$ be an orthonormal basis of the linear space which is perpendicular to $\bu$. Then, it holds
\begin{align}
\label{eqsNormal1}
    \sum_{i=1}^{d-1}\lambda_i \inner{\bz_i,\bu_j}=0\; \text{ for all $j = 2 \ldots, d$}\;.
\end{align}

By the sampling process of the $\bz_i$'s, we have that $\inner{\bz_i,\bu_j}$ for all $i=1,\ldots,d-1$ and $j=2,\ldots,d$ are i.i.d.\ standard Gaussians $\mathcal{N}(0,1)$. Hence, by foklore results, the $(d-1)\times (d-1)$ matrix $R$ consisting of the (Gaussian) entries $R_{i,j} = \inner{\bz_i,\bu_{j+1}}$ for $i \in [d-1], j \in [d-1]$ is invertible almost surely (this follows for example because the determinant is a non-zero polynomial of the entries of the matrix and by again appealing to~\cite{caron2005zero}). But \eqref{eqsNormal1} implies that $R \blambda=0$ where $\blambda=(\lambda_1, \ldots,\lambda_{d-1})^{\top}.$ Hence, from the almost sure invertibility of $R$ we conclude that necessarily $\blambda= R^{-1}\bzero = \bzero$ almost surely. But this condition readily contradicts \eqref{eqsNormal0} as we assumed that $\bu$ is of unit norm. Hence, under almost sure properties of the samples $\bz_i$, we have established the existence of the desired $\by \in \mathbb{R}^{d}$ almost surely. Now by employing the fact that $Z\bu=\bx$, we derive that for all $t \in \mathbb{R}$ and again for all $i=1,\ldots,d-1$ it holds that
\begin{align}
\label{cond1}
\inner{\bz_i,\bu+t\by}=\inner{\bz_i,\bu}=x_i.
\end{align}

Furthermore, since $\inner{\bu,\by}<0$ we have 
\begin{align}
\label{cond2}
\inf_{t>0} \|\bu+t\by\|_2=\sqrt{1-\inner{\bu,\by}^2}<1<\sup_{t>0} \|\bu+t\by\|_2=+\infty.
\end{align}
Hence, by continuity of $\|\bu + t\by\|_2$ as a function of $t \in \mathbb{R}$, there exists $t^*>0$ such that $\|\bu+t^*\by\|_2=1$. Combined with \eqref{cond1}, this implies $\bu+t^*\by \in S$.
\end{proof}

\subsection{Proof of Theorem~\ref{thm:label_rec}}
\label{sec:label-recovery-lb}
Now we present the proof of Theorem~\ref{thm:label_rec}, which establishes that one cannot information-theoretically recover the labels $\{x_i\}_{i=1}^n$ (up to a global sign flip) when $n=\lfloor \rho d \rfloor$  for any constant $\rho \in (0,1)$. This implies that our sample complexity of $n=d+1$ is optimal up to a $1+o(1)$ factor for label recovery. For the reader's convenience, we restate Theorem~\ref{thm:label_rec} below.

{
\def\thetheorem{\ref{thm:label_rec}}
\begin{theorem}[Restated]
Let $\rho \in (0,1)$ be a fixed constant and let $n = \lfloor \rho d \rfloor$.
Let $\bx \in \{-1,1\}^n$ be drawn uniformly at random, and $\bu \in \mathcal{S}^{d-1}$ be a uniformly random unit vector. For each $i=1,\ldots, n$ let $z_i \in \mathbb{R}^{d}$ be a sample generated independently from $\mathcal{N}(x_i \bu,I-\bu\bu^{\top})$. Then no estimator can exactly recover the labels $\{x_i\}_{i=1}^n$ up to a global sign flip from the observations $\{\bz_i\}_{i=1}^n$ with probability $1-o(1)$.
\end{theorem}
}

Before proving Theorem~\ref{thm:label_rec}, we present auxiliary lemmas. We first show a key lemma which characterizes the conditional distribution of $Z = (\bz_1;\ldots;\bz_n)^\top \in \mathbb{R}^{n \times d}$ given $X=\bx$, where $\bx \in \{-1,1\}^n$ is any Rademacher vector. We denote this conditional distribution by $f(Z|X=\bx)$.

Then, we define an ``energy'' function $F_H : \mathbb{R}^{n} \times \mathbb{R}^{n} \rightarrow \mathbb{R}$, where $H=ZZ^\top$, and show that it satisfies useful concentration properties when $Z$ is drawn from $f(Z|X=\bx)$ (see Definition~\ref{defn:label-energy}, and Claims~\ref{claim:energy-planted}). The energy function $F_H$ is useful because we can express the conditional density $f(Z'=Z|X=\bx)$ in terms of $F_H$ (see Remark~\ref{rem:conditional-density-energy-form}). Claim~\ref{claim:1-neighbor-energy}, which relates the value of $f(Z'=Z|X=\bx)$ to $f(Z'=Z|X=\tilde{\bx})$ for \emph{some} $\tilde{\bx} \in \{-1,1\}^n$ such that $\|\tilde{\bx}-\bx\|_0=1$, will be crucial for the proof of Theorem~\ref{thm:label_rec}.

Given the observed $Z$, we denote $Z \odot \bx:= ( x_1 z_1; \ldots ; x_n z_n)^\top \in \mathbb{R}^{n \times d} $, $\allones=(1,\ldots,1) \in \mathbb{R}^n$, and $t_+ = \max(0,t)$. We have the following characterization of the conditional distribution of $Z$ given the labels $X=\bx$.

\begin{lemma}[Conditional density given labels]
\label{lem:posterioradmiss}
The conditional distribution $\mu_{\bx}$ of $Z$ given label assignment $\bx \in \{-1, 1\}^n$ is absolutely continuous with respect to the Lebesgue measure on $\mathbb{R}^{nd}$, with a density $f(Z|X=\bx) := \frac{d\mu_{\bx}}{dZ}(Z)$ given by 
\begin{align}
\label{eq:condii}
f(Z| X = \bx) = \begin{cases}
\mathcal{Z}^{-1} \exp\left(-\frac{1}{2} \mathrm{Tr}(H)\right) \det(H)^{-\frac{1}{2}}(1-\bx^\top H^{-1} \bx)_{+}^{\frac{d-n-2}{2}}  & \text{ if } \lambda_{\text{min}}(H)>0\\
0 & \text{ otherwise}
\end{cases}\;,
\end{align}
where $H = ZZ^\top$ and $\mathcal{Z}$ is the normalization constant which does not depend on $\bx$.
\end{lemma}
The proof of Lemma~\ref{lem:posterioradmiss} can be found in Section~\ref{sec:posterior-proof}.
This lemma  establishes, via the Fisher-Neyman principle, that the Gram matrix $H$ is a sufficient statistic for the label recovery. In particular, it reveals the invariance of the model with respect to orthogonal transformations of the $d$-dimensional input (since they do not modify its Gram matrix), as already observed by \cite[Section 2.2]{unknown-cov}.

This Gram matrix will play a central role in the remainder of the proof. We now compute its conditional distribution:    
\begin{claim}[Distribution of $H$]
\label{claim:kernel-matrix-distribution} Let $\bx \in \{-1,1\}^n$ and let $Z \in \mathbb{R}^{n \times d}$ be drawn from the distribution $f(Z|X=\bx)$ (defined in Eq.~\eqref{eq:condii}). Then, the matrix $H=ZZ^\top$ is distributed as~\footnote{For random quantities $X$ and $Y$, we write $X \stackrel{\mathrm{d}}{=} Y$ to denote that $X$ and $Y$ have the same distribution.}
\begin{align}
\label{eqn:kernel-matrix}
    H \stackrel{\mathrm{d}}{=} \bx \bx^\top + Y \;,
\end{align}
where $Y = WW^\top$ for $W \in \mathbb{R}^{n \times (d-1)}$ with $W_{ij} \sim \mathcal{N}(0,1)$ for all $i \in [n], j \in [d-1]$.
\end{claim}
\begin{proof}
Let $P_x^H$ denote the sought probability measure of $H$ over the set of $n$-by-$n$ positive semi-definite matrices. 
From the definition of the Gaussian clustering model, we have that 
\begin{align}
\label{eq:basicHdist}
    P_x^H = \int_{\mathcal{S}^{d-1}} P_{x,u}^H\; \nu(du) \;,
\end{align}
where $P_{x,u}^H$ is the distribution of $H=ZZ^T$, where $Z=(z_1; \ldots; z_n)^\top \in \mathbb{R}^{n \times d}$ and $z_i$'s are independent Gaussian vectors of mean $x_i u$ and covariance $I - uu^\top$. Let $\mu_{x,u}$ denote the associated product measure on $Z$.
Fix an arbitrary $u_0 \in \mathcal{S}^{d-1}$, and observe that $Z$ has the same distribution as $Z_0 Q_u$, where $Z_0$ is drawn from $\mu_{x,u_0}$, and $Q_u$ is an orthogonal matrix which maps $u_0$ to $u$. It follows that $P_{x,u}^H$ does not depend on $u$, since for any orthogonal matrix $Q \in \mathbb{R}^{d \times d}$ $ (ZQ)(ZQ)^\top = ZZ^\top$.

Therefore, from Eq.\eqref{eq:basicHdist} we obtain that $P_x^H = P_{x,e_1}^H$. By expressing $Z$ in the canonical basis, we obtain
$z_i \stackrel{\mathrm{d}}{=} (x_i, w_i)$, with $w_{i,j} \sim \mathcal{N}(0,1)$, and therefore $H_{i,j} = \langle z_i, z_j \rangle = x_i x_j + \langle w_i, w_j \rangle$ for $i \in [n]$ and $j \in [d-1]$.
\end{proof}

By the Sherman-Morrison formula~\cite{sherman1950adjustment}, we have
\begin{align}
\label{eqn:sherman-morrison}
    H^{-1} = Y^{-1} - \frac{1}{1+\bx^\top Y^{-1} \bx}(Y^{-1} \bx)(Y^{-1} \bx)^\top\;,
\end{align}
where $Y$ follows the Wishart distribution $W_n(I_n,d-1)$. From $H^{-1}$, we define the following ``energy'' function $F_{H}$.

\begin{definition}[Energy function]
\label{defn:label-energy}
Given any positive definite $H \in \mathbb{R}^{n \times n}$, we define the energy function $F_{H}: \{-1,1\}^n \times \{-1,1\}^n \rightarrow \mathbb{R}$
\begin{align}
\label{eq:enerfuncdef}
    F_{H}(\ba,\bb) = \ba^\top H^{-1} \bb\;.
\end{align}
We abuse notation and write $F_{H}(\ba) = F_{H}(\ba,\ba)$.
\end{definition}

\begin{remark}[Conditional density using energy functions]
\label{rem:conditional-density-energy-form}
Using the energy function $F_{H}$, we can equivalently write Eq.\eqref{eq:condii} as 
\begin{align}
\label{eqn:density-given-label-energy-form}
    f(Z|X=\bx) 
    &= \mathcal{Z}^{-1}\exp\left(-\|Z\|_F^2\right) \det(ZZ^\top)^{-1/2} \cdot \left(1-F_{ZZ^\top}(\bx)\right)_{+}^{\frac{d-n-2}{2}} \;.
\end{align}
\end{remark}

\begin{claim}
For any $\ba, \bb \in \{-1,1\}^n$ and any positive definite $H \in \mathbb{R}^{n \times n}$ expressible as in Eq.~\eqref{eqn:kernel-matrix}, we have
\begin{align}
\label{eqn:random-bilinear-functional}
    F_{H}(\ba,\bb) &= \ba^\top Y^{-1}\bb - \gamma(\ba^\top Y^{-1}\bx)(\bb^\top Y^{-1}\bx)
\end{align}
where $\gamma = \frac{1}{1+\bx^\top Y^{-1} \bx} \in (0,1)$. In particular, if we take $\ba = \bb$, we have 
\begin{align}
    F_{H}(\ba)&=  \ba^\top Y^{-1}\ba - \gamma(\ba^\top Y^{-1}\bx)^2 \nonumber
\end{align}
which immediately implies $0 \prec H^{-1} \preceq Y^{-1}$, in the positive semidefinite cone order.
\end{claim}

Now, thanks to the proportional regime $n/d \to \rho\in (0,1)$ we have enough concentration to control the energy of planted labels, and, crucially, the size of the energy fluctuations for small label perturbations. This is formalised in Claims \ref{claim:energy-planted} and \ref{claim:1-neighbor-energy} respectively. 
\begin{claim}[Energy of planted labels]
\label{claim:energy-planted}
Let $d, n \in \mathbb{N}$ be such that $n = \lfloor \rho d \rfloor$ for some fixed constant $\rho \in (0,1)$. Moreover, let $F_{H}: \mathbb{R}^n \rightarrow \mathbb{R}$ be the energy function, where $H$ is drawn according to Eq.~\eqref{eqn:kernel-matrix} with planted labels $\bx$. Then, the following holds with probability $1-\exp(-\Omega(n))$ over the randomness of $Z$.
\label{eqn:energy-planted}
\begin{align*}
    F_{H}(\bx) = \frac{\bx^\top Y^{-1}\bx}{1+\bx^\top Y^{-1} \bx} \in [c,C]\;,
\end{align*}
where $0 < c < C < 1$ are fixed constants.
\end{claim}
\begin{proof}
By straightforward calculation,
\begin{align*}
    F_{H}(\bx) &= \bx^\top Y^{-1} \bx - \frac{(\bx^\top Y^{-1}\bx)^2}{1+\bx^\top Y^{-1} \bx} = \frac{\bx^\top Y^{-1}\bx}{1+\bx^\top Y^{-1} \bx}\;.
\end{align*}

Now note by~\cite[Theorem 3.3]{rudelson2010non} that with probability at least $1-\exp(-\Omega(n))$, the maximum eigenvalue of $Y^{-1}$ is of order $O(1/n)$ with probability $1-\exp(-\Omega(n))$. Hence, $\bx^\top Y^{-1} \bx \le \lambda_{\max}(Y^{-1})\|\bx\|_2^2 = O(1)$. Moreover, by~\cite[Proposition 2.4]{rudelson2010non}, $\lambda_{\min}(Y^{-1}) = \Omega(1/n)$ with probability at least $1-\exp(-\Omega(n))$. Therefore, $\bx^\top Y^{-1} \bx = \Theta(1)$ with probability $1-\exp(-\Omega(n))$, and the conclusion follows.
\end{proof}

\begin{claim}[Energy of 1-Hamming sign flip]
\label{claim:1-neighbor-energy}
Let $d, n \in \mathbb{N}$ be such that $n = \lfloor \rho d \rfloor$ for some fixed constant $\rho \in (0,1)$. Moreover, let $\bx \in \{-1,1\}^n$ and let $F_{H}: \{-1,1\}^n \rightarrow \mathbb{R}$ be the energy function, where $H$ is drawn according to Eq.~\eqref{eqn:kernel-matrix} with planted labels $\bx$. Then, there exists a constant $C > 0$ such that with probability $1-\exp(-\Omega(n))$ over the randomness of $H$, there exists $\tilde{\bx} \in \{-1,1\}^n$ with $\|\tilde{\bx} - \bx\|_0 = 1$ such that
\begin{align}
    F_{H}(\tilde{\bx}) - F_{H}(\bx) \le C/n\;.
\end{align}
\end{claim}
\begin{proof}
By simple algebraic manipulation,
\begin{align*}
    (\tilde{\bx} - \bx)^\top H^{-1}(\tilde{\bx}-\bx) &= (\tilde{\bx}-\bx)^\top H^{-1}\tilde{\bx} - (\tilde{\bx}-\bx)^\top H^{-1}\bx \\
    &= \tilde{\bx}^\top H^{-1}\tilde{\bx} - 2\tilde{\bx}^\top H^{-1} \bx + \bx^\top H^{-1}\bx \\
    &= \tilde{\bx}^\top H^{-1}\tilde{\bx} - \bx^\top H^{-1}\bx - 2(\tilde{\bx}-\bx)^\top H^{-1} \bx\;.
\end{align*}
Hence,
\begin{align*}
    F_{H}(\tilde{\bx}) - F_{H}(\bx) &= (\tilde{\bx}-\bx)^\top H^{-1} (\tilde{\bx}-\bx) + 2(\tilde{\bx}-\bx)^\top H^{-1}\bx\;.
\end{align*}
Since $\tilde{\bx}$ is a neighbor of $\bx$ with Hamming distance 1, we observe each $\tilde{\bx}$ corresponds to either $\be_i$ (or $-\be_i$) in the sense that $\tilde{\bx} - \bx = 2x_i\be_i$ for some $i \in [n]$. Using this expression, we have
\begin{align*}
    F_{H}(\tilde{\bx})-F_{H}(\bx) &= 4\be_i^\top H^{-1} \be_i - 4(x_i\be_i)^\top H^{-1} \left(\sum_{j=1}^n x_j\be_j \right) = 4F_{H}(\be_i) - 4 F_{H}(\be_i,\bx)\;,
\end{align*}
where $F_{H}(\be_i, \bx)$ is defined in \eqref{eq:enerfuncdef}.

We now observe that 
\begin{align*}
F_{H}(\bx) &= \left(\sum_{j=1}^n x_j\be_j \right)^\top H^{-1} \left(\sum_{j=1}^n x_j\be_j \right) = \sum_{j=1}^n F_{H}(x_j\be_j, \bx)\;.
\end{align*}
By an averaging argument, there exists at least one $i \in [n]$ such that $F_{H}(x_i\be_i, \bx) \ge F_{H}(\bx)/n$. By Claim~\ref{claim:energy-planted}, with probability at least $1-\exp(-\Omega(n))$, $F_{H}(\bx) \in [c_1,c_2]$ for some constants $c_1,c_2$ satisfying $0 < c_1 < c_2 < 1$. Hence, $F_{H}(x_i\be_i,\bx) \ge c_1/n$ with probability $1-\exp(-\Omega(n))$, which in particular implies $F(e_i,x) \geq 0 $ with the same probability. 

By the fact that $H^{-1} \preceq Y^{-1}$ and~\cite[Theorem 3.3]{rudelson2010non}, the maximum eigenvalue of $H^{-1}$ is $O(1/n)$ with probability $1-\exp(-\Omega(n))$. Hence, $F_{H}(\be_i) = O(1/n)$. It follows that $F_{H}(\tilde{\bx})-F_{H}(\bx) \le 4F_{H}(\be_i) - 4F_{H}(x_i\be_i,\bx) \le C/n$ for some constant $C > 0$ with probability $1-\exp(-\Omega(n))$.
\end{proof}

We are now ready to prove Theorem~\ref{thm:label_rec}.

\begin{proof}[Proof of Theorem~\ref{thm:label_rec}]

Let $\mu$ be the marginal distribution over $\mathbb{R}^{n \times d} \times \{-1,1\}^n$, and let $f(Z'= Z | X=\bx)$ denote the conditional density of $Z \in \mathbb{R}^{n \times d}$ given label $\bx \in \{-1,1\}^n$ from Lemma \ref{lem:posterioradmiss}.  Let $\hat{\bx} : \mathbb{R}^{n \times d} \rightarrow \{-1,1\}^n$ be an arbitrary (deterministic) estimator and let $\acc(\hat{\bx}) \in [0,1]$ be its accuracy, defined as follows. For any $\bx \in \{-1, 1\}^n$, 
let $S_{\hat{\bx},\bx} = \{Z \in \mathbb{R}^{n \times d} \mid \hat{\bx}(Z) = \bx \vee \hat{\bx}(Z)=-\bx\}$
denote the set of observations for which $\hat{\bx}$ is correct. Then, its accuracy (that is the probability of exactly recovering the correct sign pattern up to a global sign flip) is given by
\begin{align}
    \acc(\hat{\bx}) &= \int \one\left[\hat{\bx}(Z)=\bx \vee \hat{\bx}(Z)=-\bx\right]d\mu(\bx,Z)\nonumber \\
    &= \sum_{\bx \in \{-1,1\}^n} \int_{S_{\hat{\bx},\bx}} f(Z|\bx)\Pr[X=\bx] dZ\nonumber \\
    &= \frac{1}{2^n}\sum_{\bx \in \{-1,1\}^n} \int_{S_{\hat{\bx},\bx}} f(Z|\bx)dZ\;.
\end{align}
We show that there exists a constant $\delta > 0$ such that for any non-trivial estimator $\hat{\bx}$ achieving $\Omega(1)$ accuracy, we can construct a disjoint estimator $\hat{\by}$, i.e., $\hat{\bx}(Z) \neq \pm \hat{\by}(Z)$ for all $Z \in \mathbb{R}^{n\times d}$, such that $\acc(\hat{\by}) \ge \delta \cdot \acc (\hat{\bx})$. Since $\hat{\bx}$ and $\hat{\by}$ are disjoint, $\acc(\hat{\bx}) + \acc(\hat{\by}) \le 1$ (see \eqref{eqn:disjoint-estimators-accuracy} below). Hence, for any $\hat{\bx}: \mathbb{R}^{n \times d} \rightarrow \{-1,1\}^n$,
\begin{align*}
    (1+\delta)\acc(\hat{\bx}) \le 1 \Rightarrow \acc(\hat{\bx}) \le \frac{1}{1+\delta}\;.
\end{align*}

We note from Claim~\ref{claim:1-neighbor-energy} that for each $\bx \in \{-1,1\}^n$, with probability $1-\exp(-\Omega(n))$ over the distribution $f(Z|\bx)$, there exists a ``1-Hamming'' sign flip $\bq(\bx,Z) \in \{-1,1\}^n$ such that
\begin{align*}
    F_{H}(\bx \odot \bq(\bx,Z))-F_{H}(\bx) = O\left(\frac{1}{n}\right)\;.
\end{align*}

By Lemma~\ref{lem:posterioradmiss} and Remark~\ref{rem:conditional-density-energy-form}, we claim that the $O(1/n)$ fluctuation upper bound in $F_{H}$ implies that given $\bx \in \{-1,1\}^n$, with probability $1-\exp(-\Omega(n))$ over $f(Z|\bx)$, there exists a $1$-Hamming neighbor $\tilde{\bx}$ such that the likelihood ratio $f(Z'=Z|X=\bx)/f(Z'=Z|X=\tilde{\bx})$ is upper bounded by a constant $\delta > 0$ independent of the data dimension $d$. 

To see this, let $\bx \in \{-1,1\}^n$ be any fixed label and let $\eta$ be the random variable $\eta = 1-F_{H}(\bx)$ induced by $f(Z|\bx)$. By Claim~\ref{claim:energy-planted}, with probability $1-\exp(-\Omega(n))$, $\eta \in [c_1,c_2]$ for constants $0 < c_1 < c_2 < 1$. By Claim~\ref{claim:1-neighbor-energy} with probability $1-\exp(-\Omega(n))$, there exists a 1-Hamming sign flip $\bq(\bx,Z)$ such that $F_{H}(\bx \odot \bq(\bx,Z)) -F_{H}(\bx) \le c_3/n$ for some constant $c_3 > 0$. Hence, in combination with Lemma~\ref{lem:posterioradmiss}, the following inequality holds with probability $1-\exp(-\Omega(n))$ over $f(Z|\bx)$:
\begin{align}
    \frac{f(Z'=Z|X=\bx)}{f(Z'=Z|X=\bx \odot \bq(\bx,Z))}
    &\le \left(\frac{\eta}{\eta - c_3/n}\right)^{c_4 n} \nonumber \\
    &= \left(1+\frac{c_3/n}{\eta-c_3/n}\right)^{c_4 n} \nonumber \\
    &= \left(1+\frac{1}{(\eta/c_3)n-1}\right)^{c_4 n} \nonumber \\
    &\le \left(1+\frac{1}{c_5 n}\right)^{c_4 n} \nonumber  \\
    &\le e^{c_4/c_5}\;,
\end{align}
where we used the inequality $1+t \le e^t$ for all $t \in \mathbb{R}$. 

Let us denote $\delta := e^{-c_4/c_5}$. Then, the following holds for any $\bx \in \{-1,1\}^n$.
\begin{align}
\label{eqn:1-hamming-feasible}
    \int \one\left[\bigcup_{i=1}^n\left\{f(Z|\bx\odot \be_i) \ge \delta f(Z|\bx)\right\}\right]f(Z|\bx) dZ \ge 1-\exp(-\Omega(n))\;.
\end{align}

For each $\bx \in \{-1,1\}^n$, let $A_{\bx} \subset \mathbb{R}^{n \times d}$ denote the support of the indicator in Eq.~\eqref{eqn:1-hamming-feasible}, so
\begin{align}
    \int_{A_{\bx}^c} f(Z|\bx)dZ \le \exp(-\Omega(n))\;.
\end{align}

Moreover, the following holds for any estimator $\hat{\bx}: \mathbb{R}^{n \times d} \rightarrow \{-1,1\}^n$:
\begin{align*}
    \acc(\hat{\bx}) &= \frac{1}{2^n}\sum_{\bx \in \{-1,1\}^n} \int_{S_{\hat{\bx},\bx}} f(Z|\bx)dZ \\
    &= \frac{1}{2^n}\sum_{\bx \in \{-1,1\}^n}\int_{A_{\bx} \cap S_{\hat{\bx},\bx}} f(Z|\bx)dZ \\
    &\qquad+ \frac{1}{2^n}\sum_{\bx \in \{-1,1\}^n} \int_{A_{\bx}^c \cap S_{\hat{\bx},\bx}} f(Z|\bx)dZ\\
    &\le \frac{1}{2^n}\sum_{\bx \in \{-1,1\}^n}\int_{A_{\bx}  \cap S_{\hat{\bx},\bx}} f(Z|\bx)dZ + \exp(-\Omega(n))\;.
\end{align*}

As mentioned previously, given an estimator $\hat{\bx}$, we define a new estimator $\hat{\by}$ that is disjoint from $\hat{\bx}$. In other words, $\hat{\by}(Z) \neq \pm \hat{\bx}(Z)$ for all $Z \in \mathbb{R}^{n \times d}$. This implies $S_{\hat{\bx},\bx} \cap S_{\hat{\by},\bx} = \emptyset$ and therefore
\begin{align}
    \acc(\hat{\bx})+\acc(\hat{\by}) &= \frac{1}{2^n}\sum_{\bx \in \{-1,1\}^n} \left( \int_{S_{\hat{\bx},\bx}} f(Z|\bx)dZ + \int_{S_{\hat{\by},\bx}} f(Z|\bx)dZ \right) \nonumber \\
    &\le \frac{1}{2^n}\sum_{\bx \in \{-1,1\}^n} \int_{\mathbb{R}^{nd}} f(Z|\bx)dZ \nonumber \\
    & \le 1\;. \label{eqn:disjoint-estimators-accuracy}
\end{align}

Recall that we defined $A_{\bx}$ as the set of $Z$'s for which the following inequality between the conditional densities holds for some Hamming sign flip $\bq(\bx,Z) \in \{-1,1\}^n$:
\begin{align*}
    f(Z|\bx + \bq(\bx,Z)) \ge \delta f(Z|\bx)\;.
\end{align*}

Recall that $S_{\hat{\bx},\bx} = \{Z \in \mathbb{R}^{n \times d} \mid \hat{\bx}(Z) = \bx \vee \hat{\bx}(Z)=-\bx\}$. We consider the following partition of $S_{\hat{\bx},\bx} \cap A_{\bx}$:
\begin{align}
\label{eqn:cover-preimage}
    S_{\hat{\bx},\bx} \cap A_{\bx} = \bigcup_{i=1}^n {T}^{\bx}_{i}\;,
\end{align}
where $T_{i}^{\bx} = \{Z \in S_{\hat{\bx},\bx} \cap A_{\bx} \mid f(Z|\bx\odot \be_i) \ge \delta f(Z|\bx)\}$. Note that we can make the $T_{i}^{\bx}$'s disjoint by breaking ties between the $\be_i$'s arbitrarily.

Hence, Eq.~\eqref{eqn:cover-preimage} is indeed a partition. It follows that for each $\bx \in \{-1,1\}^n$,
\begin{align}
\label{eq:casifet}
    \sum_{i=1}^n \int_{T_{i}^{\bx}} f(Z|\bx\odot \be_i)dZ \ge \delta \int_{S_{\hat{\bx},\bx} \cap A_{\bx}} f(Z|\bx)dZ\;.
\end{align}
Now, summing over all $\bx \in \{-1,1\}^n$, we have
\begin{align*}
    \frac{1}{2^n}\sum_{\bx \in \{-1,1\}^n}\sum_{i=1}^n \int_{T_{i}^{\bx}} f(Z|\bx\odot \be_i)dZ &\ge \frac{\delta}{2^n} \sum_{\bx \in \{-1,1\}^n} \int_{S_{\hat{\bx},\bx} \cap A_{\bx}} f(Z|\bx)dZ \\
    &\ge \frac{\delta}{2^n} \sum_{\bx \in \{-1,1\}^n} \int_{S_{\hat{\bx},\bx}} f(Z|\bx)dZ - \delta \exp(-\Omega(n))\;.
\end{align*}

We define our new estimator $\hat{\by}$ such that for any $\bx \in \{-1,1\}^n$ and $i \in [n]$,
\begin{align}
\label{eqn:competitive-estimator}
    \hat{\by}(Z) = \begin{cases} \hat{\bx}(Z) \odot \be_i &\text{ for any } Z \in T_i^{\bx} \\
    \hat{\bx}(Z) \odot \be_1 &\text{ for any } Z \in \mathbb{R}^{n \times d} \setminus \bigcup_{\bx, i} T_i^{\bx}
    \end{cases}\;.
\end{align}

Now that everything is in place, let $\alpha$ be the accuracy of a given estimator $\hat{\bx}$, i.e., $\alpha = \frac{1}{2^n} \sum_{\bx \in \{-1,1\}^n} \int_{S^*_{\bx}} f(Z|\bx)dZ$. Then, our proposed disjoint estimator $\hat{\by}$ achieves accuracy at least $\delta (\alpha - \exp(-\Omega(n)))$. The two accuracies must add up to less than 1 since $\hat{\bx}$ and $\hat{\by}$ are disjoint estimators. Therefore, if $\alpha=\Omega(1)$, then for sufficiently large $n$
\begin{align*}
    \alpha + \delta (\alpha - \exp(-\Omega(n))) \le 1 \Rightarrow \alpha \le \frac{1+\delta \exp(-\Omega(n))}{1+\delta} < 1-\frac{\delta}{2(1+\delta)}=1-\Omega(1)\;.
\end{align*}Hence, we conclude that $\alpha \leq 1-\Omega(1)$ necessarily, as claimed.
\end{proof}

\subsection{Proof of Lemma \ref{lem:posterioradmiss}} 
\label{sec:posterior-proof}
{\def\thetheorem{\ref{lem:posterioradmiss}}
\begin{lemma}[Restated]
The conditional distribution $\mu_{\bx}$ of $Z$ given label assignment $\bx \in \{-1, 1\}^n$ is absolutely continuous with respect to the Lebesgue measure on $\mathbb{R}^{nd}$, with a density $f(Z|X=\bx) := \frac{d\mu_{\bx}}{dZ}(Z)$ given by 
\begin{align*}
f(Z| X = \bx) = \begin{cases}
\mathcal{Z}^{-1} \exp\left(-\frac{1}{2} \| Z \|_F^2\right) \det(H)^{-\frac{1}{2}}(1-\bx^\top H^{-1} \bx))_{+}^{\frac{d-n-2}{2}}  & \text{ if } \lambda_{\text{min}}(H)>0\\
0 & \text{ otherwise}
\end{cases}\;,
\end{align*}
where $H = ZZ^\top$ and $\mathcal{Z}$ is the normalization constant which does not depend on $\bx$.
\end{lemma}
}

\begin{proof}
Let us start by computing the joint probability distribution of $\bx$ and $Z$ by marginalizing over $\bu \in \mathcal{S}^{d-1}$. By definition, the probability distribution of $Z | \bx, \bu$, which we denote by $\mu_{\bx, \bu}$, is Gaussian with mean $(x_1 \bu, \ldots, x_n \bu) \in \mathbb{R}^{nd}$ and covariance $(I - \bu\bu^\top)^{\otimes n}$, and $u$ is uniformly distributed on $\mathcal{S}^{d-1}$, independent of $x$.  Hence, the corresponding (joint) measure of a rectangle $A \times B \times C \subseteq \RR^{nd} \times \{-1,1\}^n \times \mathcal{S}^{d-1}$ is given by
\begin{align}
\label{eq:basicfullproba}
\Pr\left\{Z \in A , \bx \in B, \bu \in C \right\} = 2^{-n} \sum_{\bx \in B} \int_{C} \mu_{\bx, \bu}(A) d\nu(\bu)\;,
\end{align}
where $\nu$ is the Haar measure on $\mathcal{S}^{d-1}$.

Now we fix an arbitrary label assignment $\bx' \in \{-1,1\}^n$, and consider the measure $\mu_{\bx'}$ on $\mathbb{R}^{nd}$ induced by marginalizing $\bu \in \mathcal{S}^{d-1}$. That is, for any (Lebesgue) measurable $A \subseteq \mathbb{R}^{nd}$,
\begin{align}
\label{eq:rur}
\mu_{\bx'}(A):= \int_{\mathcal{S}^{d-1}} \mu_{\bx',\bu}(A) d\nu(\bu)\;.
\end{align}

Let $T_{\bx}: \mathbb{R}^{n\times d} \to \mathbb{R}^{n\times d}$ be defined as $T_{\bx}(Z) = Z \odot \bx$. We verify that $\mu_{\bx'} = T_{\bx'}^{\#} \mu_{\allones}$, where $\allones$ denotes the all-ones vector and $T^{\#} \mu$ denotes the pushforward measure of $\mu$ under the mapping $T$. Hence, it suffices to consider just $\mu_{\allones}$. Observe that $\mathbf{\mu}_{\allones, \bu}$, the distribution of ${Z}| \allones, \bu$, is Gaussian with mean $(\bu, \bu, \ldots, \bu)$ and same covariance $(I - \bu\bu^\top)^{\otimes n}$. Thus,
$\mathbf{\mu}_{\allones, \bu}$ is a product measure of the form $\mathbf{\mu}_{\allones, \bu} = (\mu_u)^{\otimes n}$, where $\mu_u$ is the $(d-1)$-dimensional isotropic Gaussian measure  
supported on the hyperplane $S_{\bu} = \{ z \in \mathbb{R}^{d} \mid \langle z, \bu \rangle = 1\}$.

Let $0 < \sigma < 1/\sqrt{2}$ and consider for each $\bu \in \mathcal{S}^{d-1}$ the mollified measure $\mu_{\allones, \bu,\sigma} = (\mu_{\bu,\sigma})^{\otimes n}$, where $\mu_{\bu,\sigma}$ is now the Gaussian measure with mean $\bu$ and covariance $\Sigma =  \sigma^2 \bu\bu^T +  (I_d - \bu\bu^\top)$. We verify immediately that 
$\Sigma^{-1} = I_d + (\sigma^{-2}-1) \bu\bu^\top$. 
The density $f_{\allones, \bu, \sigma}$ of $\mu_{\allones, \bu, \sigma}$ with respect to the Lebesgue measure on $\mathbb{R}^{nd}$, which we denote by $dZ$, is thus 
\begin{align}
f_{\allones, \bu, \sigma}(Z)&:= \frac{d\mu_{\allones, \bu, \sigma}}{dZ} = (2\pi)^{-nd/2}\det(\Sigma)^{-n/2} \exp \left( \sum_{i=1}^n -\frac{1}{2}(\bz_i - \bu)^\top (I + (\sigma^{-2} - 1)\bu\bu^\top ) (\bz_i -\bu)\right) \nonumber\\
&= (2\pi)^{-nd/2} \sigma^{-n} \exp \left(-\frac{1}{2} \sum_{i=1}^n \left(\|\bz_i\|_2^2 + (\sigma^{-2} -1)\langle \bz_i \bz_i^\top, \bu\bu^\top \rangle - 2\sigma^{-2} \langle \bz_i, \bu\rangle + \sigma^{-2} \right) \right) \nonumber\\
&= \mathcal{Z}^{-1} \sigma^{-n} e^{-n/(2\sigma^2)} \exp \left(-\frac{1}{2} \left(\|Z\|_F^2  + (\sigma^{-2} -1) \bu^\top Z^\top Z \bu - 2\sigma^{-2} \langle Z^\top \allones, \bu\rangle \right) \right)\;,
\end{align}
where $\mathcal{Z}=(2\pi)^{nd/2}$ is the partition function, which does not depend on $\sigma$, $\bu$, or $Z$.

The resulting smoothed mixture $\mu_{\allones, \sigma}$ is thus given by
$\mu_{\allones, \sigma}(A) = \int_{\mathcal{S}^{d-1}} \mu_{\allones, \bu, \sigma}(A) \nu(d\bu)$ for any measurable $A \subseteq \mathbb{R}^{nd}$. Since $\mu_{\allones, \bu, \sigma}$ is absolutely continuous with respect to the Lebesgue measure on $\mathbb{R}^{nd}$ for any $\bu \in \mathcal{S}^{d-1}$, the marginalized measure $\mu_{\allones, \sigma}$ is also absolutely continuous with respect to the Lebesgue measure on $\mathbb{R}^{nd}$ with density 
\begin{align}
f_{\allones, \sigma}(Z) &:= \frac{d\mu_{\allones, \sigma}}{dZ}(Z) =  \int_{\mathcal{S}^{d-1}} f_{\allones, \bu, \sigma}(Z) d\nu(\bu) \nonumber \\
&= \mathcal{Z}^{-1}  \exp\left(-\frac{1}{2} \| {Z} \|_F^2\right) \sigma^{-n} \int_{\mathcal{S}^{d-1}} \exp \left(-\sigma^{-2} F(\bu) \right)  d\nu(\bu) \;,
\end{align}
where we define $F:\mathbb{R}^{d} \rightarrow \mathbb{R}$ to be
\begin{align}
F(u) = \frac{1}{2}(1-\sigma^2) \bu^\top Z^\top Z \bu - \langle Z^\top \allones, \bu \rangle + \frac{n}{2}~.
\end{align}

Since $Z^\top Z$ has rank $n<d$, we can marginalize over the remaining $d-n$ variables. Indeed, let $Z = V \Lambda U^\top$ be the singular value decomposition of $Z$, where $U \in \mathbb{R}^{d \times n}$ and $V \in \mathbb{R}^{n \times n}$ are orthogonal matrices and $\Lambda = \text{diag}(\lambda_1, \ldots, \lambda_n)$ with $\lambda_1 \ge \ldots \ge \lambda_n \ge 0$. Note that $Z^\top Z = U \Lambda^2 U^\top$. Now consider the projection $\bv = U^\top \bu \in \mathcal{B}_{n}$, where $\mathcal{B}_{n}$ is the $n$-dimensional $\ell_2$ unit ball, 
and $\br = U^\top Z^\top \allones = \Lambda V^\top \allones \in \mathbb{R}^n$. Then, $F(\bu) = E(\bv)$, where $E: \mathbb{R}^{n} \rightarrow \mathbb{R}$ is given by
\begin{align}
\label{eq:enerraw}
E(\bv) = \frac{1}{2}(1-\sigma^2)\bv^\top \Lambda^2 \bv - \langle \br,  \bv \rangle + \frac{n}{2}  \;. 
\end{align}

Since $\bu$ is uniformly distributed in $\mathcal{S}^{d-1}$, the joint distribution of $\bv = U^\top \bu$ is spherically symmetric, and the squared radius is distributed according to $\mathrm{Beta}(n/2, (d-n)/2)$~\cite[Remark 2.10]{diaconis1987dozen}. We refer to the proof of~\cite[Lemma 4]{naor2003projecting} for a detailed derivation of this fact. Hence,
\begin{align}
f_{\allones, \sigma}(Z) &= \mathcal{Z}^{-1}  \exp\left(-\frac{1}{2} \| Z \|_F^2\right) \sigma^{-n} \int_{\mathcal{B}_{n}} \exp( -\sigma^{-2}E({\bv})) (1-\|\bv\|_{2}^2)^{\frac{d-n-2}{2}} d\bv \nonumber \\
&= \mathcal{Z}^{-1} \exp\left(-\frac{1}{2} \| Z \|_F^2 \right) \sigma^{-n}  \int_{\mathcal{B}_{n}} \exp( -\sigma^{-2}{E}_1(\bv)+{E}_2(\bv)) (1-\|\bv\|_{2}^2)^{\frac{d-n-2}{2}} d\bv \;, \label{eqn:integral-n-dim-ball}
\end{align}
where we ${E}_1$ and ${E}_2$ comes from the following decomposition of $E$.
\begin{align}
E_1(\bv) = \frac{1}{2}\bv^\top \Lambda^2 \bv - \langle \br, \bv \rangle + \frac{n}{2},\; E_2(\bv) = \frac{1}{2}\bv^\top \Lambda^2 \bv\;.
\end{align}

We now apply Laplace's approximation method \cite[Chapter 5]{breitung2006asymptotic} to estimate 
the integral 
\begin{align}
\label{eq:yutt}
\sigma^{-n}\int_{\mathcal{B}_{n}} \exp\left( -\sigma^{-2}{E}_1({\bv})+{E}_2(\bv))\right) (1-\|\bv\|_{2}^2)^{\frac{d-n-2}{2}} d\bv\;.    
\end{align}

Without loss of generality, we assume that $Z$ is such that $\lambda_{\mathrm{min}}(Z Z^\top) = \lambda_{n}^2>0$ since the event $\lambda_{\mathrm{min}}(ZZ^\top) = 0$ has zero Lebesgue measure.
Since ${E}_1$ is quadratic and $\Lambda^2 \succ 0$, $E_1$ is strongly convex. 
Let ${\bv}^*$ be the (unique) global minimum of ${E}_1$, which is 
given by ${\bv}^* = \Lambda^{-2} \br = \Lambda^{-1} V^\top \allones$. 
\begin{align*}
    {E}_1({\bv}^*) = -\frac{1}{2} \br^\top \Lambda^{-2} \br + \frac{n}{2}= -\frac{1}{2} (\allones)^\top V V^\top (\allones) + \frac{n}{2}= -\frac{1}{2} \|\allones\|_2^2 + \frac{n}{2} = 0\;.
\end{align*}

Furthermore,
\begin{align*}
    {E}_2({\bv}^*) = \frac{1}{2} \br^\top \Lambda^{-2} \br = \frac{1}{2} (\allones)^\top V V^\top (\allones) = \frac{1}{2} \|\allones\|_2^2 = \frac{n}{2}\;.
\end{align*}

We now distinguish two cases, depending whether ${\bv}^*$ lies in the interior of $\mathcal{B}_{n}$ or not. 
If ${\bv}^* \in (\mathcal{B}_{n})^\circ$, then by Laplace's approximation 
we have, for all $Z$ such that $\lambda_{\text{min}}(Z Z^\top)>0$, 
\begin{align}
\sigma^{-n}\int_{\mathcal{B}_{n}} \exp\left( -\sigma^{-2}{E}_1({\bv})+{E}_2({\bv}))\right) &(1-\|{\bv}\|_2^2)^{\frac{d-n-2}{2}} d\bv \\
&\stackrel{\sigma\to 0}{\to} \left(2\pi\right)^{n/2} \frac{\exp\left(-\sigma^{-2}{E}_1({\bv}^*) + {E}_2({\bv}^*)\right) (1-\|{\bv}^*\|_2^2)^{\frac{d-n-2}{2}}}{|\nabla^2 {E}_1({\bv}^*)|^{1/2}} \nonumber\\
&=\left(2\pi\right)^{n/2} \frac{\exp\left(-\sigma^{-2}{E}_1({\bv}^*) + {E}_2({\bv}^*)\right) (1-\|{\bv}^*\|_2^2)^{\frac{d-n-2}{2}}}{\sqrt{\det(Z Z^\top)}} \nonumber \\
&= \left(2\pi\right)^{n/2} \frac{e^{n/2} (1-\|{\bv}^*\|_2^2)^{\frac{d-n-2}{2}}}{\sqrt{\det(Z Z^\top)}}\;.
 \end{align}

Let us now show that for any $Z \in \mathbb{R}^{n \times d}$ such that $\lambda_{\text{min}}(Z Z^\top)>0$ and $u^* \notin (\mathcal{B}_{n})^\circ$, then 
\begin{align}
\label{eq:mingot}
\sigma^{-n}\int_{\mathcal{B}_{n}} \exp( -\sigma^{-2}{E}_1({\bv}))\exp({E}_2({\bv})) (1-\|{\bv}\|_2^2)^{\frac{d-n-2}{2}} d\bv \rightarrow 0\;,\text{ as } \sigma \to 0~.    
\end{align}

Indeed, notice that Eq.\eqref{eq:yutt} can be equivalently written as
\begin{align*}
\sigma^{-n}\int_{\mathbb{R}^n} \exp\left( -\sigma^{-2}{E}_1({\bv})+{E}_2({\bv})\right) (1-\|{\bv}\|_2^2)_{+}^{\frac{d-n-2}{2}} d\bv\;,
\end{align*}
where we denote $t_+ = \max(0,t)$. 
Observing that $(1-\|{\bv}\|_2^2)_{+}^{\frac{d-n-2}{2}}=0$ whenever ${\bv}^* \notin (\mathcal{B}_{n})^\circ$, we conclude that the leading order term in the Laplace approximation vanishes, thereby proving (\ref{eq:mingot}).  

We have just shown that, as $\sigma \to 0$, the density $f_{\allones, \sigma}$ of $\mu_{\allones, \sigma}$ admits a pointwise limit almost everywhere, i.e., $f_{\allones, \sigma} \stackrel{\text{a.e.}}{\rightarrow} f_{\allones}$ where the limit $f_{\allones}$ is given explicitly by 
\begin{align}
\label{eq:grit}
f_{\allones}(Z) = \mathcal{Z}^{-1} \exp\left(-\frac{1}{2} \| Z \|_F^2\right) \frac{ \left(1-{(\allones)}^\top (ZZ^\top)^{-1} \allones\right)_{+}^{\frac{d-n-2}{2}}}{\sqrt{\det(Z Z^\top)}}~,~\text{for } Z \,\text{ s.t. }\lambda_{\text{min}}(Z Z^\top)>0\;.
 \end{align}

Let us now show that the sequence $f_{\allones, \sigma}$ is dominated by some function $g$ in $L^1(\mathbb{R}^{nd})$, the set of all Lebesgue-integrable functions on $\mathbb{R}^{nd}$. To this end, observe first that 
\begin{align}
\label{eq:rebasic}
\sigma^{-n}\int_{\mathcal{B}_{n}} \exp\left( -\sigma^{-2}{E}_1({\bv})+{E}_2(\bv))\right) (1-\|\bv\|_{2}^2)^{\frac{d-n-2}{2}} d\bv &\leq \sigma^{-n} \int_{\mathbb{R}^n} \exp\left( -\sigma^{-2}{E}({\bv}))\right) d\bv\;.
\end{align}

Since $E(v)$ is quadratic in $v$ with a symmetric positive-definite Hessian, the RHS is a Gaussian integral, which we can evaluate exactly. From the identity 
\begin{align*}
    \int_{\mathbb{R}^n} \exp\left( -\frac{1}{2}v^\top H v + \langle \beta, v \rangle + c\right) dv = (2\pi)^{n/2} (\det{H})^{-1/2}\exp\left( c + \frac{1}{2}\beta^\top H^{-1} \beta\right) \;,
\end{align*}
and the definition of $E(v)$ in Eq.\eqref{eq:enerraw}, we obtain 
\begin{align}
\label{eq:rebasic2}
\sigma^{-n} \int_{\mathbb{R}^n} \exp\left( -\sigma^{-2}{E}({\bv}))\right) d\bv 
&= \sigma^{-n} \left(\frac{2\pi}{\sigma^{-2} -1}\right)^{n/2} (\det{\Lambda})^{-1} \exp\left(-\frac{1}{2\sigma^2}\left(n - \frac{1}{1-\sigma^2} r^\top \Lambda^{-2} r\right) \right)  \nonumber  \\
&= \left(\frac{2\pi}{1-\sigma^2}\right)^{n/2} (\det{\Lambda})^{-1} \exp\left(\frac{n}{2(1-\sigma^2)}\right)~,
\end{align}
by recalling that $r^\top \Lambda^{-2} r = n$. Thus, for any $\sigma < 1/\sqrt{2}$, from Eq.\eqref{eq:rebasic} and Eq.\eqref{eq:rebasic2}, we have 
\begin{align}
  \sigma^{-n}\int_{\mathcal{B}_{n}} \exp\left( -\sigma^{-2}{E}_1({\bv})+{E}_2(\bv))\right) (1-\|\bv\|_{2}^2)^{\frac{d-n-2}{2}} d\bv \leq (4\pi)^{n/2} e^n (\det{\Lambda})^{-1}\;.  
\end{align}

Hence, the following holds for any $\sigma < 1/\sqrt{2}$. 
\begin{align}
f_{\allones, \sigma}(Z) \leq C_n \det(ZZ^\top)^{-1/2} \exp\left(-\frac{1}{2} \| Z \|_F^2 \right) := C_n g(Z)~,
\end{align}
where $C_n$ is a constant that depends only on $n$. We now show that $g \in L^{1}(\mathbb{R}^{nd})$. 
 
Indeed, let $Z = L U^\top$ be the LQ decomposition of $Z$, where $L$ is a lower triangular $n \times n$ matrix and $U \in \mathbb{R}^{d \times n}$ is orthogonal, belonging to the Stiefel manifold $\mathcal{V}(n,d)$. Furthermore, to ensure uniqueness of the LQ decomposition, we assume that the diagonal entries of $L$ are positive. Denote by $T_{n}$ the space of such triangular matrices.   
By \cite[Theorem 2.1.13]{muirhead2009aspects}, the differential volume element $dZ$ is expressed in the LQ decomposition as 
\begin{align}
dZ = \left|\prod_{i=1}^n L_{i,i}^{d-i} \right| dL dU\;.
\end{align}
As a result, we have 
\begin{align}
\label{eq:domin}
\int g(Z) dZ 
&\leq \int_{\mathbb{R}^{n \times d}} \exp\left(-\frac{1}{2} \| Z \|_F^2\right) \det(Z Z^\top)^{-1/2} dZ  \\
&= \int_{T_n} \int_{\mathcal{V}(n,d)} \exp\left(-\frac{1}{2} \| L \|_F^2\right)  \det(L)^{-1} \left|\prod_{i=1}^n L_{i,i}^{d-i} \right| dL dU \nonumber \\ 
&= \int_{T_n} \int_{\mathcal{V}(n,d)} \exp\left(-\frac{1}{2} \| L \|_F^2\right)  \det(L)^{d-n-1} \left|\prod_{i=1}^n L_{i,i}^{n-i} \right| dL dU \nonumber \\
&\le \mu(\mathcal{V}(n,d)) \int_{T_n} \exp\left(-\frac{1}{2} \| L \|_F^2\right)  \|L \|_F^{nd} dL \nonumber \\
&< \infty \nonumber \;,
\end{align}
where we used $\det(Z Z^\top)^{1/2} =|\prod_{i=1}^n L_{ii}|$, the fact that $\mathcal{V}(n,d)$ is a compact manifold with finite Haar measure, and that polynomial moments of the Gaussian distribution exist for any fixed order. 

We therefore obtain from Eq.\eqref{eq:domin} that the sequence $f_{\allones, \sigma}$ is dominated and converges pointwise a.e. to $f_{\allones}$. As a result, by the dominated convergence theorem~\cite[Theorem 1.19]{evans2018measure}, we obtain that $f_{\allones}$ is integrable, with $\int f_{\allones}(Z) dZ = 1$, and
\begin{align}
\label{eq:huhu}
\lim_{\sigma \to 0} \int \left| f_{\allones, \sigma} - f_{\allones}\right| dZ = 0\;.
\end{align}

Let $\tilde{\mu}$ be the measure induced by the density $f_{\allones}$. That is, $\tilde{\mu}(A) = \int_A f_{\allones}(Z)dZ$, 
where $f_{\allones}(Z)=0$ for $Z \in \mathbb{R}^{nd}$ such that $\lambda_{\text{min}}(Z^\top Z)=0$.
As a consequence of Eq.\eqref{eq:huhu}, for any measurable set $A \subseteq \mathbb{R}^{nd}$,
\begin{align}
|\tilde{\mu}(A) - \mu_{\allones, \sigma}(A) | 
&= \left|\int_A f_{\allones}(Z) - f_{\allones, \sigma}(Z)dZ \right| \\
&\le \int_{\mathbb{R}^{nd}} |f_{\allones}(Z) - f_{\allones, \sigma}(Z)| dZ ~\to 0 ~,~\text{ as } \sigma \to 0 \nonumber\,
\end{align}
which shows that $\mu_{\allones, \sigma}$ weakly converges to $\tilde{\mu}$.

Finally, we argue that the measure $\mu_{\allones}$ (defined in Eq.\eqref{eq:rur}) necessarily admits the desired density with respect to the Lebesgue measure on $\mathbb{R}^{nd}$.
Since for each $u \in \mathcal{S}^{d-1}$, the distributions $\mu_{\allones, \bu}$ and $\mu_{\allones, u, \sigma}$ share the same mean, and have covariances 
$(\sigma^2 uu^T +  (I - uu^\top))^{\otimes n}$ and $(I - uu^\top)^{\otimes n}$ that commute, it follows from \cite[Theorem 4]{olkin1982distance} that their 2-Wasserstein distance satisfies
\begin{align*}
W_2(\mu_{\allones, \bu} , \mu_{\allones, u, \sigma}) = \left\| \left((\sigma^2 uu^T +  (I - uu^\top))^{\otimes n}\right)^{1/2} - \left((I - uu^\top)^{\otimes n}\right)^{1/2} \right \|_F^2 = O(\sigma^2)\;.
\end{align*}

It follows that 
\begin{align*}
W_2( \mu_{\allones}, \mu_{\allones, \sigma}) \leq \int_{\mathcal{S}^{d-1}} W_2(\mu_{\allones, \bu} , \mu_{\allones, u, \sigma}) d\nu(\bu) = O(\sigma^2)\;, 
\end{align*}
and therefore that as $\sigma \to 0$, the smoothed measure $\mu_{\allones, \sigma}$ converges weakly to $\mu_{\allones}$, since $W_2$ metrizes weak convergence on Euclidean spaces \cite[Theorem 6.9]{villani2009optimal}. Since $\mu_{\allones, \sigma}$ weakly converges to both $\mu_{\allones}$ and $\tilde{\mu}$, we conclude that $\mu_{\allones} = \tilde{\mu}$, and the proof is complete.
\end{proof}

\section*{Acknowledgments}

The authors would like to thank Jonathan Niles-Weed, Oded Regev, Kaizheng Wang, Tselil Schramm, Ilias Diakonikolas, Elchanan Mossel and Nike Sun for helpful discussions during different stages of this project.

IZ is supported by the Simons-NSF grant DMS-2031883 on the Theoretical Foundations of Deep Learning  and the Vannevar Bush Faculty Fellowship ONR-N00014-20-1-2826. MS and JB are partially supported by the Alfred P. Sloan Foundation, NSF RI-1816753, NSF CAREER CIF-1845360, and DMS MoDL Scale 2134216. ASW is supported by a Simons-Berkeley Research Fellowship. Part of this work was done while IZ and ASW were visiting the Simons Institute for the Theory of Computing during Fall 2021.

\bibliographystyle{alpha}
\bibliography{main}

\end{document}